\RequirePackage[hyphens]{url}

\documentclass[runningheads]{llncs}
\usepackage{graphicx}
\usepackage{amsmath,amssymb} 
\usepackage{color}

\usepackage{soul}
\usepackage{mathtools}
\usepackage{algorithm}
\usepackage{algorithmic}
\usepackage{enumitem}
\usepackage{url}
\usepackage{subfigure}
\usepackage{comment}
\usepackage{caption}

\DeclarePairedDelimiter\floor{\lfloor}{\rfloor}

\newcommand{\ba}{\mathbf{a}}

\newcommand{\bu}{\mathbf{u}}

\newcommand{\bP}{\mathbf{P}}

\newcommand{\bs}{\mathbf{s}}

\newcommand{\bx}{\mathbf{x}}
\newcommand{\by}{\mathbf{y}}

\newcommand{\mbbR}{\mathbb{R}}
\newcommand{\mbbI}{\mathbb{I}}

\DeclarePairedDelimiter{\norm}{\lVert}{\rVert}

\begin{document}
\title{Deterministic Consensus Maximization with Biconvex Programming} 

\titlerunning{Deterministic Consensus Maximization with Biconvex Programming}
%
\author{Zhipeng Cai\inst{1} \and
Tat-Jun Chin\inst{1} \and
Huu Le\inst{2} \and
David Suter\inst{3}}
%
\index{Chin, Tat-Jun}
\authorrunning{Z. Cai, T.-J. Chin, H. Le and D. Suter}
%

\institute{School of Computer Science, The University of Adelaide \\ \email{ \{zhipeng.cai,tat-jun.chin\}@adelaide.edu.au}\and
	School of Electrical Engineering and Computer Science, Queensland University of Technology\\
	\email{huu.le@qut.edu.au}\and
School of Computing and Security, Edith Cowan University\\
	\email{d.suter@ecu.edu.au}}
\maketitle              
\begin{abstract}
Consensus maximization is one of the most widely used robust fitting paradigms in computer vision, and the development of algorithms for consensus maximization is an active research topic. In this paper, we propose an efficient \emph{deterministic optimization} algorithm for consensus maximization. Given an initial solution, our method conducts a \emph{deterministic search} that forcibly increases the consensus of the initial solution. We show how each iteration of the update can be formulated as an instance of biconvex programming, which we solve efficiently using a novel biconvex optimization algorithm. In contrast to our algorithm, previous consensus improvement techniques rely on random sampling or relaxations of the objective function, which reduce their ability to significantly improve the initial consensus. In fact, on challenging instances, the previous techniques may even return a worse off solution. Comprehensive experiments show that our algorithm can consistently and greatly improve the quality of the initial solution, without substantial cost.\footnote{Matlab demo program is available at \url{https://github.com/ZhipengCai/Demo---Deterministic-consensus-maximization-with-biconvex-programming}.}

\end{abstract}
\section{Introduction}

Due to the existence of noise and outliers in real-life data, robust model fitting is necessary to enable many computer vision applications. Arguably the most prevalent robust technique is random sample consensus (RANSAC)~\cite{fischler1981random}, which aims to find the model that has the largest consensus set. The RANSAC algorithm approximately solves this optimization problem, by repetitively sampling minimal subsets of the data, in the hope of ``hitting" an all-inlier minimal subset that gives rise to a model hypothesis with high consensus.

Many variants of RANSAC have been proposed~\cite{choi09}. Most variants attempt to conduct guided sampling using various heuristics, so as to speed up the retrieval of all-inlier minimal subsets. Fundamentally, however, taking minimal subsets reduces the span of the data and produces biased model estimates~\cite{meer04robust,tran2014sampling}. Thus, the best hypothesis found by RANSAC often has much lower consensus than the maximum achievable, especially on higher-dimensional problems. In reality, the RANSAC solution should only be taken as a rough initial estimate~\cite{chum2003locally}.

To ``polish" a rough RANSAC solution, one can perform least squares (LS) on the consensus set of the RANSAC estimate (i.e.~the Gold Standard Algorithm~\cite[Chap.~4]{hartley2003multiple}). Though justifiable from a maximum likelihood point of view, the efficacy of LS depends on having a sufficiently large consensus set to begin with.

A more useful approach is Locally Optimized RANSAC (LO-RANSAC)~\cite{chum2003locally,lebeda2012fixing}, which attempts to enlarge the consensus set of an initial RANSAC estimate, by generating hypotheses from \emph{larger-than-minimal subsets} of the consensus set.\footnote{This is typically invoked from within a main RANSAC routine.} The rationale is that hypotheses fitted on a larger number of inliers typically lead to better estimates with even higher support. Ultimately, however, LO-RANSAC is also a randomized algorithm. Although it conducts a more focused sampling, the algorithm cannot guarantee improvements to the initial estimate. As we will demonstrate in Sec.~\ref{sec:exp_geo}, often on more challenging datasets, LO-RANSAC is unable to significantly improve upon the RANSAC result.

Due to its combinatorial nature, consensus set maximization is NP-hard~\cite{chin2018robust}. While this has not deterred the development of globally optimal algorithms~\cite{olsson2008polynomial,zheng2011deterministically,li2009consensus,enqvist2015tractable,chin15,chin2016guaranteed,specialeconsensus,campbell2017iccv}, the fundamental intractability of the problem means that global algorithms are essentially variants of exhaustive search-and-prune procedures, whose runtime scales exponentially in the general case. While global algorithms have their place in computer vision, currently they are mostly confined to problems with low-dimensions and/or small number of measurements.

\subsection{Deterministic algorithms---a new class of methods}

Recently, efficient deterministic algorithms for consensus maximization are gaining attention~\cite{leexact,purkait2017emmcvpr}. Different from random sampling, such algorithms begin with an initial solution (obtained using least squares or a random sampling method) and iteratively performs \emph{deterministic updates} on the solution to improve its quality. While they do not strive for the global optimum, such algorithms are able to find excellent solutions due to the directed search.

To perform deterministic updating, the previous methods relax the objective function (Le et al.~\cite{leexact} use $\ell_1$ penalization, and Purkait et al.~\cite{purkait2017emmcvpr} use a smooth surrogate function). Invariably this necessitates the setting of a smoothing parameter that controls the degree of relaxation, and the progressive tightening of the relaxation to ensure convergence to a good solution. As we will demonstrate in Sec.~\ref{sec:worseOff}, incorrect settings of the smoothing parameter and/or its annealing rate may actually lead to a worse solution than the starting point.

\subsection{Our contributions}

We propose a novel deterministic optimization algorithm for consensus maximization. The overall structure of our method is a bisection search to increase the consensus of the current solution. The key to the effectiveness of our method is to formulate the feasibility test in each iteration as a \emph{biconvex program}, which we solve efficiently via a biconvex optimization algorithm. Unlike~\cite{leexact,purkait2017emmcvpr}, our method neither relaxes the objective function, nor requires tuning of smoothing parameters. On both synthetic and real datasets, we demonstrate the superior performance of our method over previous consensus improvement techniques.

\section{Problem definition}

Given a set of $N$ outlier contaminated measurements, consensus maximization aims to find the model $\bx\in D$ that is consistent with the largest data subset
\begin{align}
&\underset{\bx \in D}{\text{maximize}}\ \  \mathcal{I}(\bx),\label{eq:maxcon}
\end{align}
where $D$ is the domain of model parameters (more details later), and
\begin{align}
&\mathcal{I}(\bx) = \sum_{i=1}^N\mbbI\left(r_i(\bx) \le \epsilon\right)\label{eq:maxconConstraint}
\end{align}
counts the number of inliers (consensus) of $\bx$. Function $r_i(\bx)$ gives the \emph{residual} of the i-th measurement w.r.t.~$\bx$, $\epsilon$ is the inlier threshold and $\mbbI$ is the indicator function which returns 1 if the input statement is true and 0 otherwise.

Fig.~\ref{fig:updateStep} illustrates the objective function $\mathcal{I}(\bx)$. As can be appreciated from the inlier counting operations, $\mathcal{I}(\bx)$ is a step function with uninformative gradients.

\begin{figure*}[b]\centering
	\subfigure{\includegraphics[width=1\columnwidth]{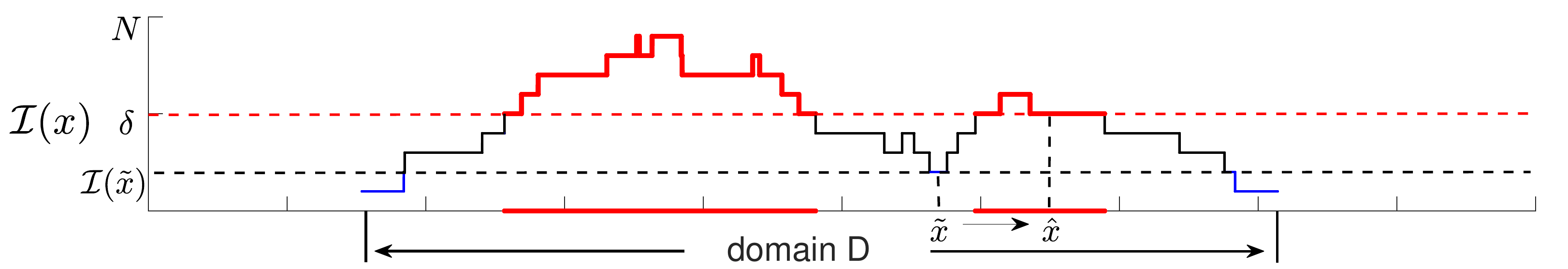}}
	\caption{Illustrating the update problem. Given the current solution $\tilde{\bx}$ and a target consensus $\delta$, where $\delta > \mathcal{I}(\tilde{\bx})$, the update problem~\eqref{eq:ref} aims to find another solution $\hat{\bx}$ with $\mathcal{I}(\hat{\bx}) \ge \delta$. Later in Sec.~\ref{sec:main}, problem~\eqref{eq:ref} will be embedded in a broader algorithm that searches over $\delta$ to realize deterministic consensus maximization.}
	\label{fig:updateStep}
\end{figure*}

\subsection{The update problem}

Let $\tilde{\bx}$ be an initial solution to~\eqref{eq:maxcon}; we wish to improve $\tilde{\bx}$ to yield a better solution. We define this task formally as
\begin{align}\label{eq:ref}
\begin{aligned}
&\text{find} & \bx\in D, && \text{such that} && \mathcal{I}(\bx)\geq \delta,
\end{aligned}
\end{align}
where $\delta > \mathcal{I}(\tilde{\bx})$ is a target consensus value. See Fig.{~\ref{fig:updateStep}} for an illustration. For now, assume that $\delta$ is given; later in Sec.~\ref{sec:main} we will embed~\eqref{eq:ref} in a broader algorithm to search over $\delta$.

Also, although~\eqref{eq:ref} does not demand that the revised solution be ``close" to $\tilde{\bx}$, it is strategic to employ $\tilde{\bx}$ as a starting point to perform the update. In Sec.~\ref{sec:continuousReformulation}, we will propose such an algorithm that is able to efficiently solve~\eqref{eq:ref}.

\subsection{Residual functions and solvable models}\label{sec:models}

Before embarking on a solution for~\eqref{eq:ref}, it is vital to first elaborate on the form of $r_i(\bx)$ and the type of models that can be fitted by the proposed algorithm. Following previous works~\cite{kahl2008multiple,ke2007quasiconvex,chin2016guaranteed}, we focus on residual functions of the form
\begin{align}\label{eq:generalReprojError}
r_i(\bx) = \frac{q_i(\bx)}{p_i(\bx)},
\end{align}
where $q_i(\bx)$ is convex quadratic and $p_i(\bx)$ is linear. We also insist that $p_i(\bx)$ positive. We call $r_i(\bx)$ the \emph{quasiconvex geometric residual} since it is quasiconvex~\cite[Sec.~3.4.1]{boyd2004convex} in the domain
\begin{align}\label{eq:domain}
D = \{ \bx \in \mathbb{R}^d \mid p_i(\bx) > 0, i = 1,\dots,N \},
\end{align}
Note that $D$ in the above form specifies a convex domain in $\mathbb{R}^d$.

Many model fitting problems in computer vision have residuals of the type~\eqref{eq:generalReprojError}. For example, in multiple view triangulation where we aim to estimate the 3D point $\bx \in \mathbb{R}^3$ from multiple (possibly incorrect) 2D observations $\{ \bu_i \}^{N}_{i=1}$,
\begin{align}
r_i(\bx) = \frac{\| (\bP^{(1:2)}_{i} - \bu_i \bP^{(3)}_i) \bar{\bx}  \|_2}{\bP^{(3)}_i\bar{\bx}}
\end{align}
is the reprojection error in the $i$-th camera, where $\bar{\bx} = [\bx^T~1]^T$,
\begin{align}
\bP_i = \left[ \begin{matrix} \bP_i^{(1:2)} \\ \bP_i^{(3)} \end{matrix} \right] \in \mathbb{R}^{3\times 4}
\end{align}
is the $i$-th camera matrix with $\bP_i^{(1:2)}$ and $\bP_i^{(3)}$ respectively being the first-two rows and third row of $\bP$. Insisting that $\bx$ lies in the convex domain $D = \{ \bx \in \mathbb{R}^3 \mid \bP^{(3)}_i\bar{\bx} > 0, \forall i \}$ ensures that the estimated $\bx$ lies in front of all the cameras. 

Other model fitting problems with quasiconvex geometric residuals include homography fitting, camera resectioning, and the known rotation problem; see~\cite{kahl2008multiple} for details and other examples. However, note that fundamental matrix estimation is not a quasiconvex problem~\cite{kahl2008multiple}; in Sec.~\ref{sec:exp}, we will show how the proposed technique can be adapted to robustly estimate the fundamental matrix.

\section{Solving the update problem}\label{sec:continuousReformulation}

As the decision version of~\eqref{eq:maxcon}, the update problem~\eqref{eq:ref} is NP-complete~\cite{chin2018robust} and thus can only be approximately solved. In this section, we propose an algorithm that works well in practice, i.e., able to significantly improve $\tilde{\bx}$.

\subsection{Reformulation as continuous optimization}

With quasicovex geometric residuals~\eqref{eq:generalReprojError}, the inequality $r_i(\bx) \le \epsilon$ becomes
\begin{align}\label{eq:thresholding}
q_i(\bx)-\epsilon p_i(\bx) \le 0.
\end{align}
Since $q_i(\bx)$ is convex and $p_i(\bx)$ is linear, the constraint~\eqref{eq:thresholding} specifies a convex region in $D$. Defining
\begin{align}\label{eq:r_i'}
r_i^\prime(\bx) := q_i(\bx)-\epsilon p_i(\bx)
\end{align}
and introducing for each $r_i'(\bx)$ an indicator variable $y_i\in [0,1]$ and a slack variable $s_i \ge 0$, we can write~\eqref{eq:ref} using complementarity constraints~\cite{hu2012linear} as
\begin{subequations}\label{eq:ref_comlementarity}
\begin{align}
&\text{find}
& & \bx \in D\label{eq:ref_comlementarity(a)}\\
&\text{subject to}
& & \sum_i  y_i \geq \delta,	\label{eq:ref_comlementarity(b)}\\
&
& & y_i \in [0, 1], \;\; \forall i,  \label{eq:ref_comlementarity(c)}\\
&
& & y_i s_i = 0, \;\; \forall i, \label{eq:ref_comlementarity(d)}\\
&
& & s_i - r_i'(\bx) \geq 0, \;\; \forall i, \label{eq:ref_comlementarity(e)}\\
&
& & s_i \geq 0, \;\; \forall i.\label{eq:ref_comlementarity(f)}
\end{align}
\end{subequations} 
Intuitively, $y_i$ reflects whether the i-th datum is an inlier w.r.t.~$\bx$.  In the following, we establish the integrality of $y_i$ and the equivalence between~\eqref{eq:ref_comlementarity} and~\eqref{eq:ref}.

\begin{lemma}
Problems~\eqref{eq:ref_comlementarity} and~\eqref{eq:ref} are equivalent.
\end{lemma}
\begin{proof}
Observe that for any $\bx$,
\begin{itemize}
\item[\textbf{a1:}] If $r_i'(\bx) > 0$, the $i$-th datum is outlying to $\bx$, and~\eqref{eq:ref_comlementarity(d)} and~\eqref{eq:ref_comlementarity(e)} will force $s_i \geq r_i'(\bx)>0$ and $y_i = 0$.
\item[\textbf{a2:}] If $r_i'(\bx) \leq 0$, the $i$-th datum is inlying to $\bx$, and~\eqref{eq:ref_comlementarity(f)} and~\eqref{eq:ref_comlementarity(d)} allow $s_i$ and $y_i$ to have only one of the following settings: \textbf{a2.1:} $s_i > 0$ and $y_i = 0$; or \textbf{a2.2:} $s_i = 0$ and $y_i$ being indeterminate.
\end{itemize}
If $\bx$ is infeasible for~\eqref{eq:ref}, i.e., $\mathcal{I}(\bx) < \delta$, condition \textbf{a1} ensures that~\eqref{eq:ref_comlementarity(b)} is violated, hence $\bx$ is also infeasible for~\eqref{eq:ref_comlementarity}. Conversely, if $\bx$ is infeasible for~\eqref{eq:ref_comlementarity}, i.e., $\sum_i{y_i} < \delta$, then $\mathcal{I}(\bx) < \delta$, hence $\bx$ is also infeasible for~\eqref{eq:ref}.

If $\bx$ is feasible for~\eqref{eq:ref}, we can always set $y_i = 1$ and $s_i = 0$ for all inliers to satisfy~\eqref{eq:ref_comlementarity(b)}, ensuring the feasibility of $\bx$ to~\eqref{eq:ref_comlementarity}. Conversely, if  $\bx$ is feasible for~\eqref{eq:ref_comlementarity}, by \textbf{a1} there are at least $\delta$ inliers, thus $\bx$ is also feasible to~\eqref{eq:ref}.\qed
\end{proof}

From the computational standpoint,~\eqref{eq:ref_comlementarity} is no easier to solve than~\eqref{eq:ref}. However, by constructing a cost function from the bilinear constraints~\eqref{eq:ref_comlementarity(d)}, we arrive at the following continuous optimization problem
\begin{subequations}\label{eq:ref_continuous}
\begin{align}
&\underset{\bx \in D,~\bs \in \mbbR^{N},~\by \in \mbbR^{N}}{\text{minimize}}
& & \sum_{i}{y_i s_i} \label{eq:ref_continuous(a)}\\
&\text{subject to}
& & \sum_i y_i \geq \delta,	\label{eq:ref_continuous(b)}\\
&
& & y_i \in [0, 1], \;\; \forall i,  \label{eq:ref_continuous(c)}\\
&
& & s_i - r_i'(\bx) \geq 0, \;\; \forall i, \label{eq:ref_continuous(d)}\\
&
& & s_i \geq 0, \;\; \forall i, \label{eq:ref_continuous(e)}
\end{align}
\end{subequations} 
where $\bs = \left[ s_1, \dots, s_N \right]^T$ and $\by = \left[ y_1, \dots, y_N \right]^T$. The following lemma establishes the equivalence between~\eqref{eq:ref_continuous} and~\eqref{eq:ref}.
\begin{lemma}\label{lem:equivalence2}
If the globally optimal value of~\eqref{eq:ref_continuous} is zero, then there exists $\bx$ that satisfies the update problem~\eqref{eq:ref}.
\end{lemma}
\begin{proof}
Due to~\eqref{eq:ref_continuous(c)} and~\eqref{eq:ref_continuous(e)}, the objective value of~\eqref{eq:ref_continuous} is lower bounded by zero. Let $(\bx^\ast, \bs^\ast, \by^\ast)$ be a global minimizer of~\eqref{eq:ref_continuous}. If $\sum_i y_i^\ast s_i^\ast = 0$, then $\bx^\ast$ satisfies all the constraints in~\eqref{eq:ref_comlementarity}, thus $\bx^\ast$ is feasible to~\eqref{eq:ref}.\qed
\end{proof}

\subsection{Biconvex optimization algorithm}\label{sec:BCO}

Although all the constraints in~\eqref{eq:ref_continuous} are convex (including $\bx \in D$), the objective function is not convex. Nonetheless, the primary value of formulation~\eqref{eq:ref_continuous} is to enable the usage of convex solvers to approximately solve the update problem. Note also that~\eqref{eq:ref_continuous} does not require any smoothing parameters. 

To this end, observe that~\eqref{eq:ref_continuous} is in fact an instance of \emph{biconvex programming}~\cite{wiki:Biconvex}. If we fix $\bx$ and $\bs$,~\eqref{eq:ref_continuous} reduces to the linear program (LP)
\begin{subequations}\label{eq:solveY}
\begin{align}
&\underset{\by\in\mbbR^N}{\text{minimize}} 
& &\sum_i{y_i s_i}\\
&\text{subject to} 
& &\sum_i{y_i}\geq \delta,\label{eq:solveY(b)}\\
& 
& &y_i\in[0,1],\ \  \forall i,
\end{align}
\end{subequations}
which can be solved in close form.\footnote{Set $y_i = 1$ if $s_i$ is one of the $\delta$-smallest slacks, and $y_i = 0$ otherwise.} On the other hand, if we fix $\by$,~\eqref{eq:ref_continuous} reduces to the second order cone program (SOCP)
\begin{subequations}\label{eq:solveSX}
\begin{align}
&\underset{\bx\in D, \bs\in\mbbR^{N}}{\text{minimize}} 
& &\sum_{i}{ y_i s_i}\label{eq:solveSX(a)}\\
&\text{subject to}
& & s_i - r_i'(\bx) \geq 0, \;\; \forall i,\label{eq:solveSX(b)}\\
&
& & s_i \geq 0, \;\; \forall i.\label{eq:solveSX(e)}
\end{align}
\end{subequations}
Note that $s_i$ does not have influence if the corresponding $y_i = 0$; these slack variables can be removed from the problem to speed up optimization.\footnote{Given the optimal $\hat{\bx}$ for~\eqref{eq:solveSX}, the values of the slack variables that did not participate in the problem can be obtained as $s_i = \max\{0,r_i'(\hat{\bx})\}$.}

The proposed algorithm (called \emph{Biconvex Optimization} or \emph{BCO}) is simple: we initialize $\bx$ as the starting $\tilde{\bx}$ from~\eqref{eq:ref}, and set the slacks as
\begin{align}\label{eq:initS}
&s_i = \max{\{0,r_i'(\tilde{\bx})\}},\;\; \forall i.
\end{align}
Then, we alternate between solving the LP and SOCP until convergence. Since~\eqref{eq:ref_continuous} is lower-bounded by zero, and each invocation of the LP and SOCP are guaranteed to reduce the cost, BCO will always converge to a \emph{local optimum} $(\hat{\bx}, \hat{\bs}, \hat{\by})$.

\begin{algorithm}[t]\centering
\caption{Biconvex optimization (BCO) for the continuous problem~\eqref{eq:ref_continuous}.}
\label{alg:BCO}                         
\begin{algorithmic}[1]      
\REQUIRE Initial solution $\tilde{\bx}$, target consensus $\delta$.
\STATE Initialize $\hat{\bx} \leftarrow \tilde{\bx}$, set $\hat{\bs}$ using~\eqref{eq:initS}.
\WHILE {not converged}
\STATE $\hat{\by} \leftarrow$ solve LP~\eqref{eq:solveY}.
\STATE $(\hat{\bx},\hat{\bs})\leftarrow$ solve SOCP~\eqref{eq:solveSX}.
\ENDWHILE
\RETURN $\hat{\bx}$, $\hat{\bs}$ and $\hat{\by}$.
\end{algorithmic}
\end{algorithm}

In respect to solving the update problem~\eqref{eq:ref}, if the local optimum $(\hat{\bx}, \hat{\bs}, \hat{\by})$ turns out to be the global optimum (i.e., $\sum_i \hat{y}_i \hat{s}_i = 0$), then $\hat{\bx}$ is a solution to~\eqref{eq:ref}, i.e., $\mathcal{I}(\hat{\bx}) \ge \delta$. Else, $\hat{\bx}$ might still represent an improved solution over $\tilde{\bx}$. Compared to randomized search, our method is by design more capable of improving $\tilde{\bx}$. This is because optimizing~\eqref{eq:ref_continuous} naturally reduces the residual of outliers that ``should be'' an inlier, i.e., with $y_i = 1$, which may still lead to a local refinement, i.e., $\mathcal{I}(\hat{\bx})>\delta_l = \mathcal{I}(\tilde{\bx})$, regardless of whether problem~\eqref{eq:ref} is feasible or not. In the next section, we will construct an effective deterministic consensus maximization technique based on Algorithm~\ref{alg:BCO}.

\section{Main algorithm---deterministic consensus maximization}\label{sec:main}

Given an initial solution $\bx^{(0)}$ to~\eqref{eq:maxcon}, e.g., obtained using least squares or a random sampling heuristic, we wish to update $\bx^{(0)}$ to a better solution.  The main structure of our proposed algorithm is simple: we conduct bisection over the consensus value to search for a better solution; see Algorithm~\ref{alg:Bisec}.

\begin{algorithm}[t]\centering
\caption{Bisection (non-global) for deterministic consensus maximization.}
\label{alg:Bisec}                         
\begin{algorithmic}[1]                   
\REQUIRE Initial solution $\bx^{(0)}$ for~\eqref{eq:maxcon} obtained using least squares or random sampling.
\STATE $\tilde{\bx} \leftarrow \bx^{(0)}$, $\delta_{h} \leftarrow N$, $\delta_{l} \leftarrow \mathcal{I}(\bx^{(0)})$.
\WHILE {$\delta_{h}>\delta_{l}+1$}
\STATE $\delta \leftarrow \floor{0.5(\delta_l + \delta_h)}$.
\STATE $(\hat{\bx}, \hat{\bs}, \hat{\by}) \leftarrow$ BCO($\tilde{\bx}$, $\delta$) (see Algorithm~\ref{alg:BCO}).\label{step:bco}
\IF{$\mathcal{I}(\hat{\bx}) > \mathcal{I}(\tilde{\bx})$}
\STATE $\tilde{\bx} \leftarrow \hat{\bx}$, $\delta_l \leftarrow \mathcal{I}(\hat{\bx})$.
\ENDIF
\IF{$\mathcal{I}(\hat{\bx})<\delta$}
\STATE $\delta_h \leftarrow \delta$.
\ENDIF
\ENDWHILE
\RETURN $\tilde{\bx}$, $\delta_l$.
\end{algorithmic}
\end{algorithm}

A lower and upper bound $\delta_l$ and $\delta_h$ for the consensus, which are initialized respectively to $\mathcal{I}(\bx^{(0)})$ and $N$, are maintained and progressively tightened. Let $\tilde{\bx}$ be the current best solution (initialized to $\bx^{(0)}$); then, the midpoint $\delta = \lfloor 0.5(\delta_l + \delta_h) \rfloor$ is obtained and the update problem via the continuous biconvex formulation~\eqref{eq:ref_continuous} is solved using Algorithm~\ref{alg:BCO}. If the solution $\hat{\bx}$ for~\eqref{eq:ref_continuous} has a higher quality than the incumbent, $\tilde{\bx}$ is revised to become $\hat{\bx}$ and $\delta_l$ is increased to $\mathcal{I}(\hat{\bx})$. And if $\mathcal{I}(\hat{\bx})<\delta$, $\delta_h$ is decreased to $\delta$. Algorithm~\ref{alg:Bisec} ends when $\delta_h = \delta_l+1$.

Since the ``feasibility test" in Algorithm~\ref{alg:Bisec} (Step~\ref{step:bco}) is solved via a non-convex subroutine, the bisection technique does not guarantee finding the global solution, i.e., the quality of the final solution may underestimate the maximum achievable quality. However, our technique is fundamentally advantageous compared to previous methods~\cite{chum2003locally,leexact,purkait2017emmcvpr} since it is not subject to the vagaries of randomization or require tuning of hyperparameters. Empirical results in the next section will demonstrate the effectiveness of the proposed algorithm.

\section{Results}\label{sec:exp}

We call the proposed algorithm \emph{IBCO} (for \emph{iterative biconvex optimization}). We compared IBCO against the following random sampling methods:
\begin{itemize}[leftmargin=1em]
  \item RANSAC (RS)~\cite{fischler1981random} (baseline): the confidence $\rho$ was set to $0.99$ for computing the termination threshold.
  \item PROSAC (PS)~\cite{chum2005matching} and Guided MLESAC (GMS)~\cite{tordoff2005guided} (RS variants with guided sampling): only tested for fundamental matrix and homography estimation since inlier priors like matching scores for correspondences were needed.
  \item LO-RANSAC (LRS)~\cite{chum2003locally}: subset size in inner sampling was set to half of the current consensus size, and the max number of inner iterations was set to $10$.
  \item Fixing LO-RANSAC (FLRS)~\cite{lebeda2012fixing}: subset size in inner sampling was set to $7 \times$ minimal subset size, and the max number of inner iterations was set to $50$.
  \item USAC~\cite{raguram13}: a modern technique that combines ideas from PS and LRS.\footnote{Code from \url{htts://http://www.cs.unc.edu/~rraguram/usac/USAC-1.0.zip}.} USAC was evaluated only on fundamental matrix and homography estimation since the available code only implements these models.
\end{itemize} 
Except USAC which was implemented in C++, the other sampling methods were based on MATLAB~\cite{KovesiMATLABCode}. Also, least squares was executed on the final consensus set to refine the results of all the random sampling methods.

In addition to the random sampling methods, we also compared IBCO against the following deterministic consensus maximization algorithms:
\begin{itemize}[leftmargin=1em]
 \item Exact Penalty (EP) method~\cite{leexact}: The method\footnote{Code from \url{https://cs.adelaide.edu.au/~huu/}.} was retuned for best performance on our data: we set the penalty parameter $\alpha$ to 1.5 for fundamental matrix estimation and 0.5 for all other problems. The annealing rate $\kappa$ for the penalty parameter was set to 5 for linear regression and 2D homography estimation and 1.5 for triangulation and fundamental matrix estimation.
  \item Smooth Surrogate (SS) method~\cite{purkait2017emmcvpr}: Using our own implementation. The smoothing parameter $\gamma$ was set to $0.01$ as suggested in~\cite{purkait2017emmcvpr}.
\end{itemize}
For the deterministic methods, Table~\ref{tab:solver4optimization} lists the convex solvers used for their respective subproblems. Further, results for these methods with both FLRS and random initialization ($\bx^{(0)}$ was generated randomly) were provided, in order to show separately the performance with good (FLRS) and bad (random) initialization. We also tested least squares initialization, but under high outlier rates, its effectiveness was no better than random initialization. All experiments were executed on a laptop with Intel Core 2.60GHz i7 CPU and 16GB RAM. 

\begin{table}[t]
\begin{center}
\begin{tabular}{ |c|c|c|c|c|c| }
\hline
Convex subproblem & LP & SOCP \\ 
 \hline
 Solvers used & Gurobi & Sedumi \\  
 \hline
 Methods using the solver & EP, SS & IBCO \\ 
 \hline 
\end{tabular}
\end{center}
\caption{Convex solvers used in deterministic methods.}\label{tab:solver4optimization}
\end{table}

\subsection{Robust linear regression on synthetic data}\label{sec:exp_lin}
Data of size $N = 1000$ for $8$-dimensional linear regression (i.e., $\bx \in \mathbb{R}^8$) were synthetically generated. In linear regression, the residual takes the form
\begin{align}
r_i(\bx) = \norm{\ba_i^T\bx-b_i}_2,
\end{align}
which is a special case of~\eqref{eq:generalReprojError} (set $p_i(\bx) = 1$), and each datum is represented by \{$\ba_i\in\mbbR^8$, $b_i \in \mbbR$\}. First, the independent measurements $\{ \ba_i \}^{N}_{i=1}$ and parameter vector $\bx$ were randomly sampled. The dependent measurements were computed as $b_i = \ba_i^T \bx$ and added with noise uniformly distributed between $[-0.3,0.3]$. A subset of $\eta\%$ of the dependent measurements were then randomly selected and added with Gaussian noise of $\sigma = 1.5$ to create outliers. To guarantee the outlier rate, each outlier is regenerated until the noise is not within [-0.3,0.3]. The inlier threshold $\epsilon$ for~\eqref{eq:maxcon} was set to $0.3$. 

Fig.~\ref{fig:LinResult} shows the optimized consensus, runtime and model accuracy of the methods for $\eta \in \{0, 5,...,70, 75\}$, averaged over $10$ runs for each data instance. Note that the actual outlier rate was sometimes slightly lower than expected since the largest consensus set included some outliers with low noise value. For $\eta = 75$ the actual outlier rate was around 72\% (see Fig.~\ref{subfig:LinCon}). To prevent inaccurate analysis caused by this phenomenon, results for $\eta>75$ were not provided. 

\begin{figure}[t]\centering
	\subfigure[Average optimized consensus.]{\includegraphics[width=0.49\columnwidth]{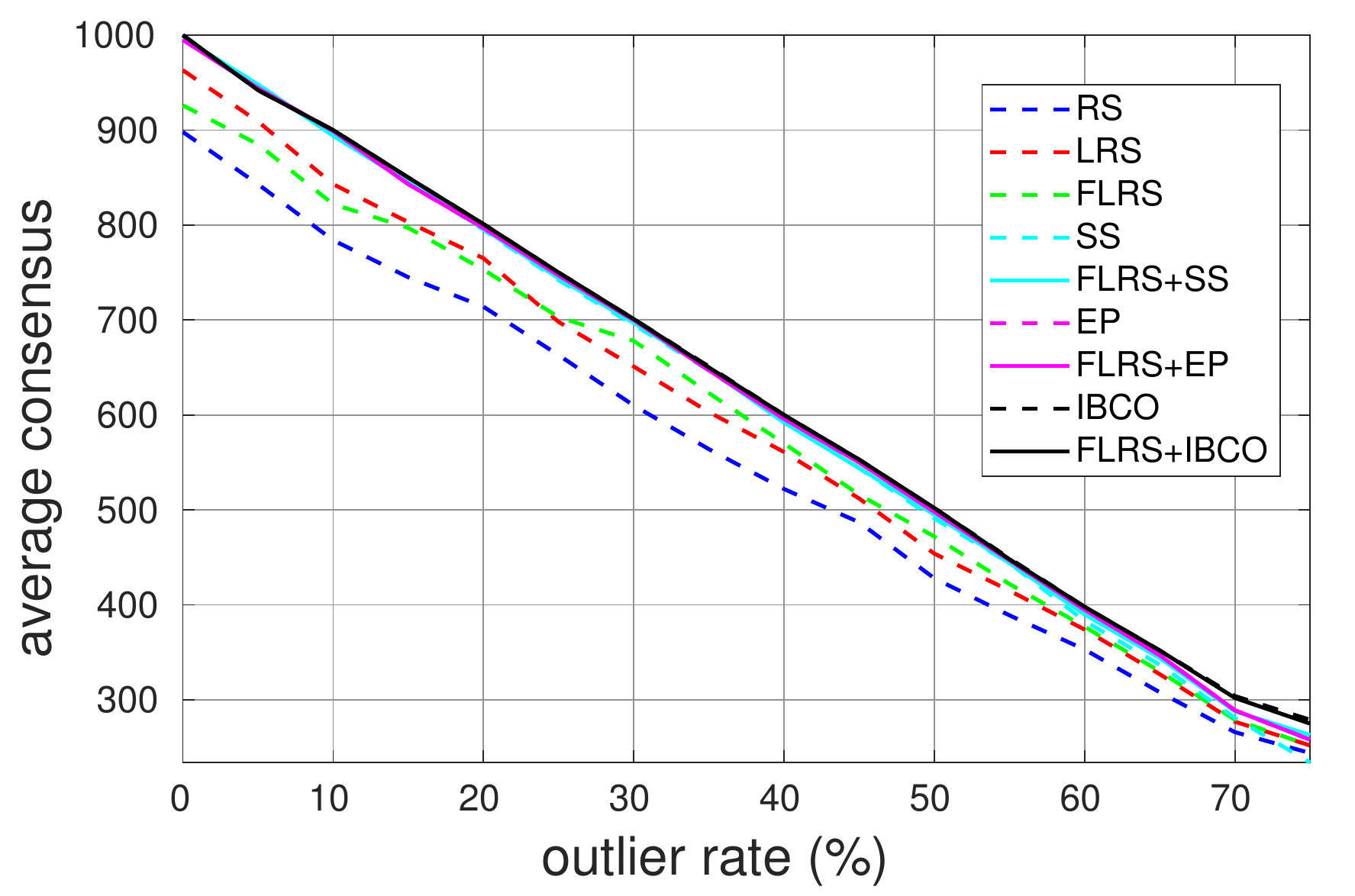}\label{subfig:LinCon}}
	\subfigure[Relative difference of consensus to RS.]{\includegraphics[width=0.49\columnwidth]{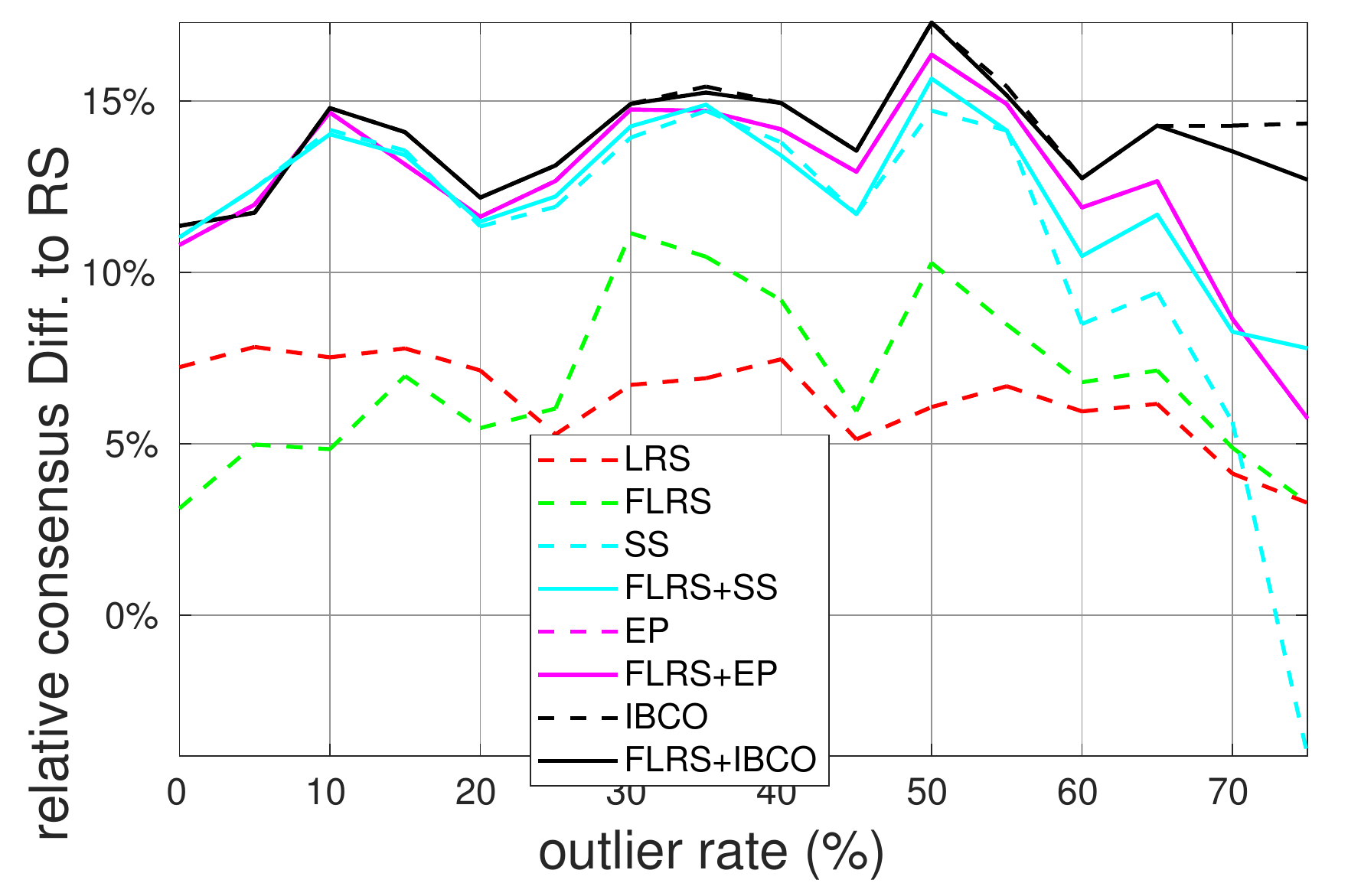}\label{subfig:LinRef}}
	\subfigure[Average runtime (seconds, in log scale).]{\includegraphics[width=0.49\columnwidth]{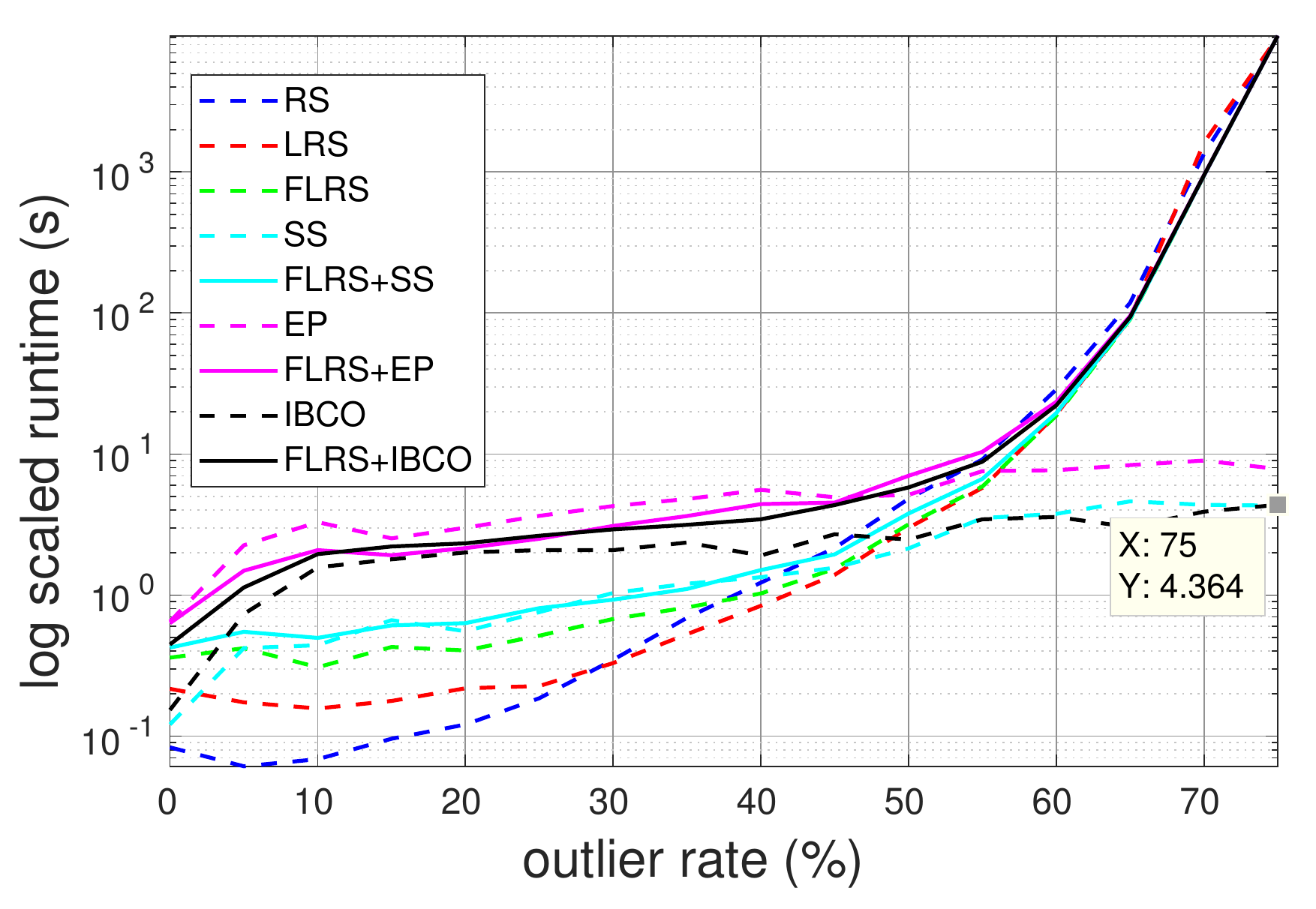}\label{subfig:LinTime}}
	\subfigure[Average residuals on ground truth inliers for models fitted on the consensus set by least squares.]{\includegraphics[width=0.49\columnwidth]{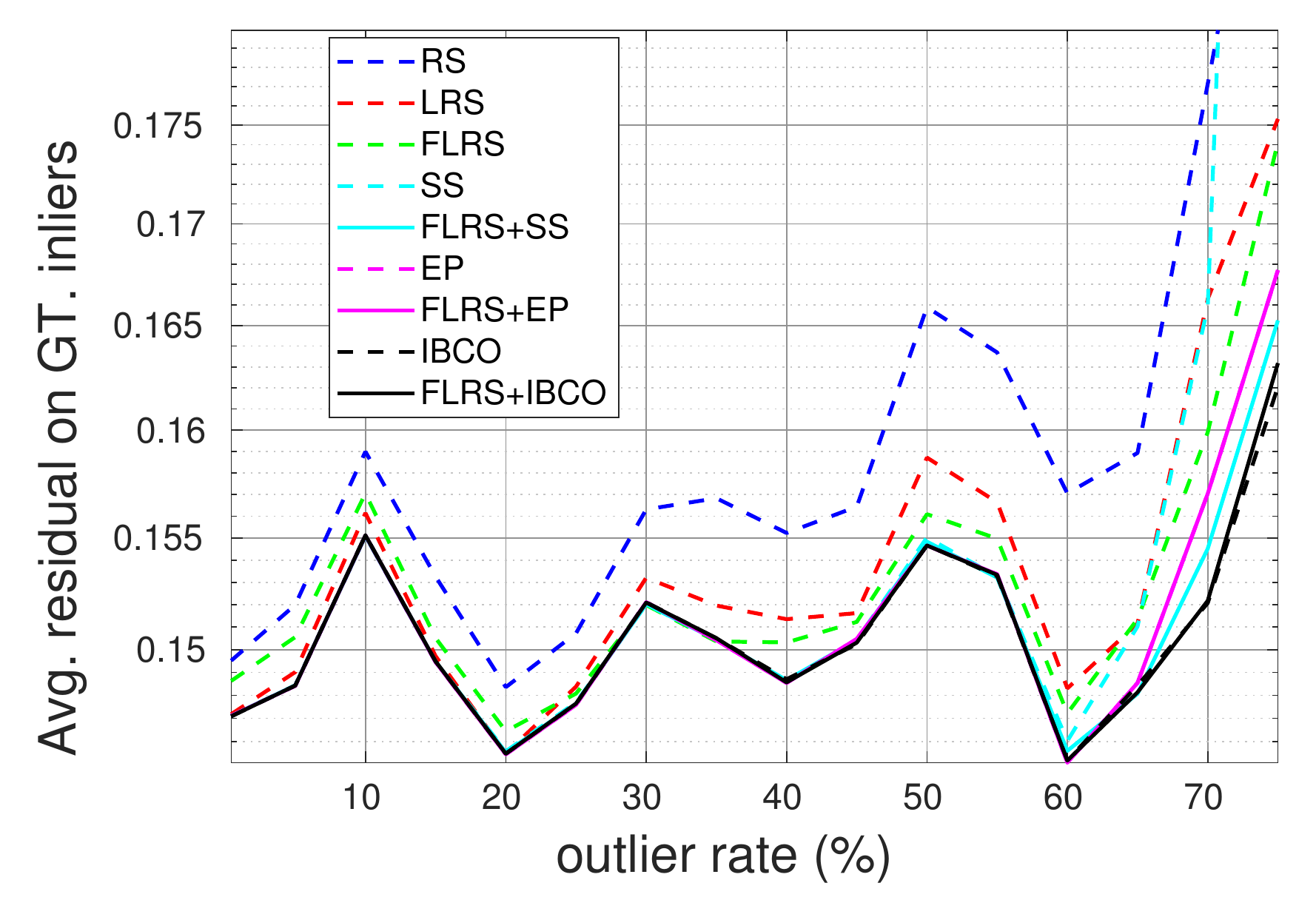}\label{subfig:LinResGT}}
	\caption{Robust linear regression results with varied $\eta$ (approx. outlier rate).}
	\label{fig:LinResult}
\end{figure}

Fig.~\ref{subfig:LinRef} demonstrates for each method the relative consensus difference to RS. It is evident that both IBCO variants outperformed other methods in general. Unlike other methods, whose improvement to RS was low at high outlier rates, both IBCO variants were consistently better than RS by more than 11\%. Though IBCO was only marginally better than EP for outlier rates lower than 65\%, Fig.~\ref{subfig:LinCon} shows that for most of the data instances, both IBCO variants found consensus very close or exactly equal to the maximum achievable. The cost of IBCO was fairly practical (less than 5 seconds for all data instances, see the data tip in Fig.~\ref{subfig:LinTime}). Also the runtime of the random sampling methods (RS, LRS, FLRS) increased exponentially with $\eta$. Hence, at high $\eta$, the major cost of FLRS+EP, FLRS+SS and FLRS+IBCO came from FLRS. 

To demonstrate the significance of having higher consensus, we further performed least squares fitting on the consensus set of each method. Given a least squares fitted model $\bx_{LS}$, define the average residual on ground truth inliers (the data assigned with less than 0.3 noise level) as:
\begin{align}
& e(\bx_{LS}) = \frac{\sum_{i^\ast\in \mathcal{I}^\ast}{r_{i^\ast}(\bx_{LS})}}{|\mathcal{I^\ast}|},
\end{align}
where $\mathcal{I}^\ast$ was the set of all ground truth inliers.
Fig.~\ref{subfig:LinResGT} shows $e(\bx_{LS})$ for all methods on all data instances. Generally, higher consensus led to a lower average residual, suggesting a more accurate model.

\subsection{Homography estimation}\label{sec:exp_geo}

Five image pairs from the NYC Library dataset~\cite{wilson_eccv2014_1dsfm} were used for 2D homography estimation. On each image pair, SIFT correspondences were produced by the VLFeat toolbox~\cite{vedaldi08vlfeat} and used as inputs. Fig.~\ref{fig:homoData} depicts examples of inputs, as well as consensus sets from FLRS and FLRS+IBCO. The transfer error in one image~\cite[Sec. 4.2.2]{hartley2003multiple} was used as the distance measurement. The inlier threshold $\epsilon$ was set to 4 pixels. The 4-Point algorithm~\cite[Sec. 4.7.1]{hartley2003multiple} was used in all random sampling approaches for model fitting on minimal samples. 

\begin{figure}[!htb]\centering
	\subfigure[Input correspondences ($N = 455$).]{\includegraphics[width=0.32\columnwidth]{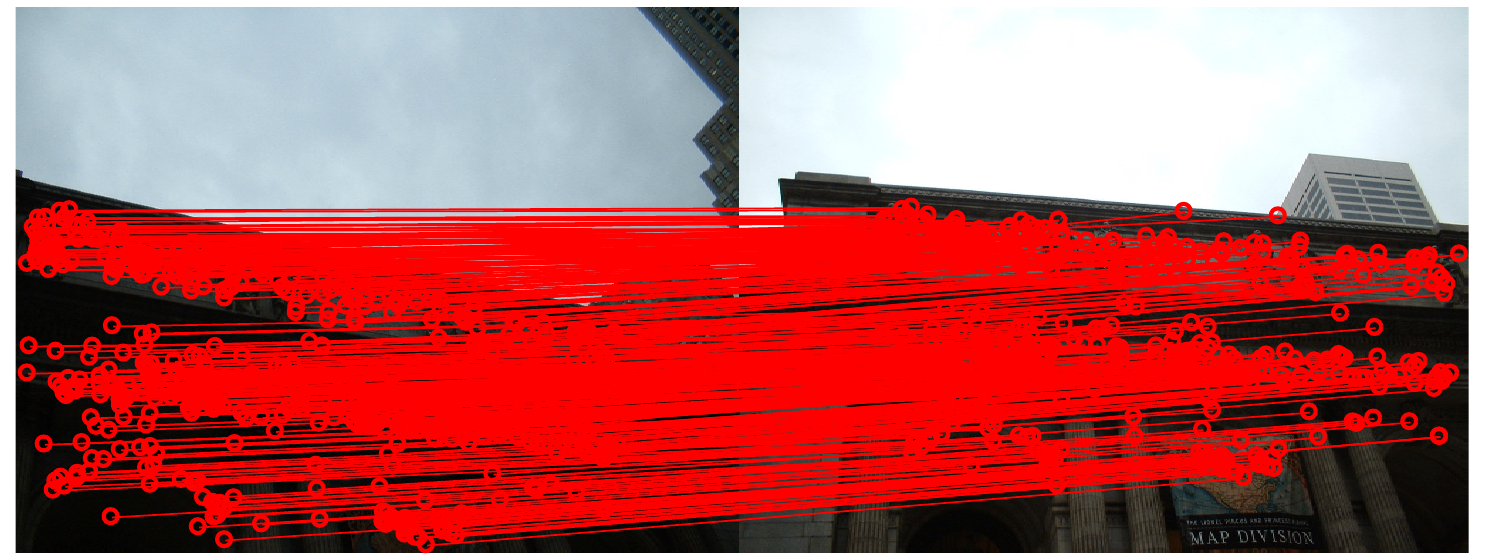}\label{subfig:Building1Corr}}
	\subfigure[FLRS consensus set (consensus: 323).]{\includegraphics[width=0.32\columnwidth]{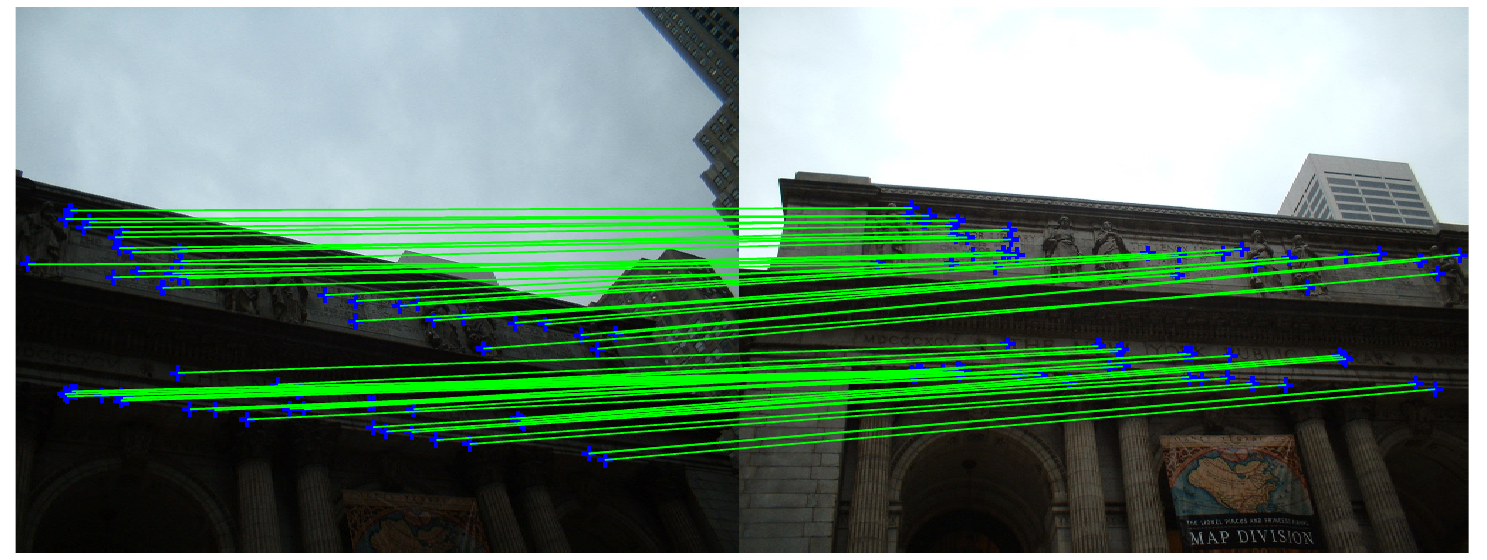}\label{subfig:Building2FLRS}}
		\subfigure[FLRS + IBCO consensus set (consensus: 353).]{\includegraphics[width=0.32\columnwidth]{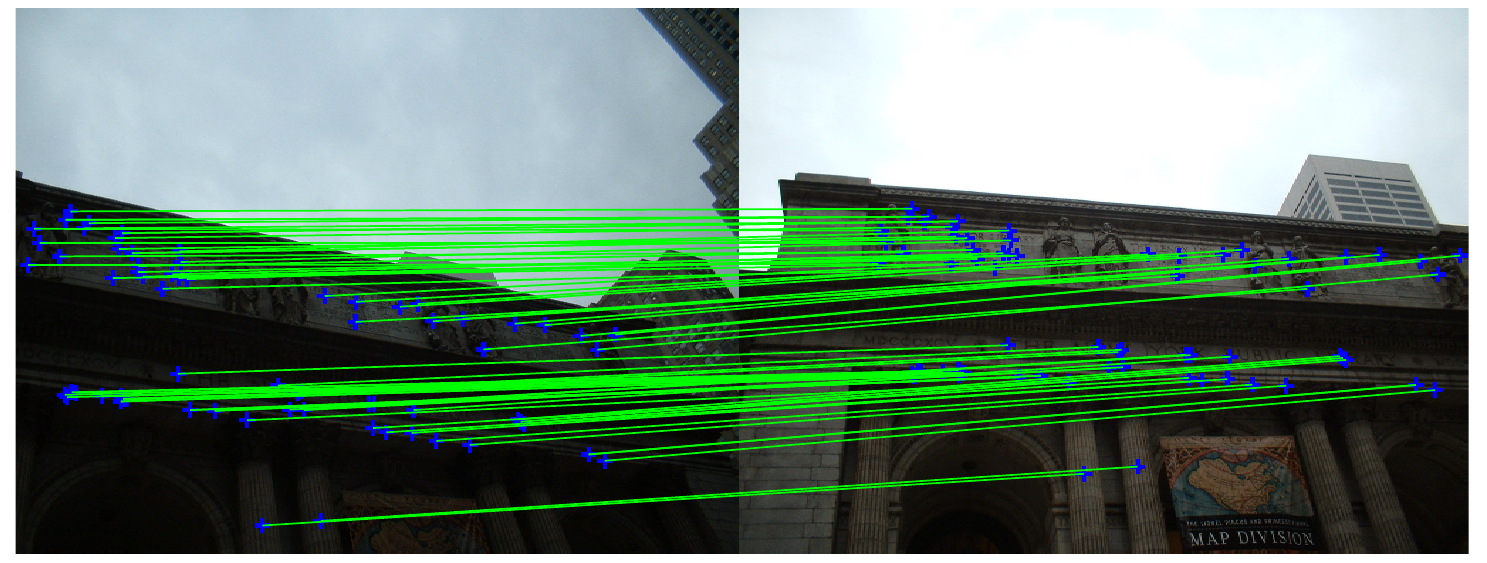}\label{subfig:Building1IBCO}}
	\subfigure[Input correspondences ($N = 346$).]{\includegraphics[width=0.32\columnwidth]{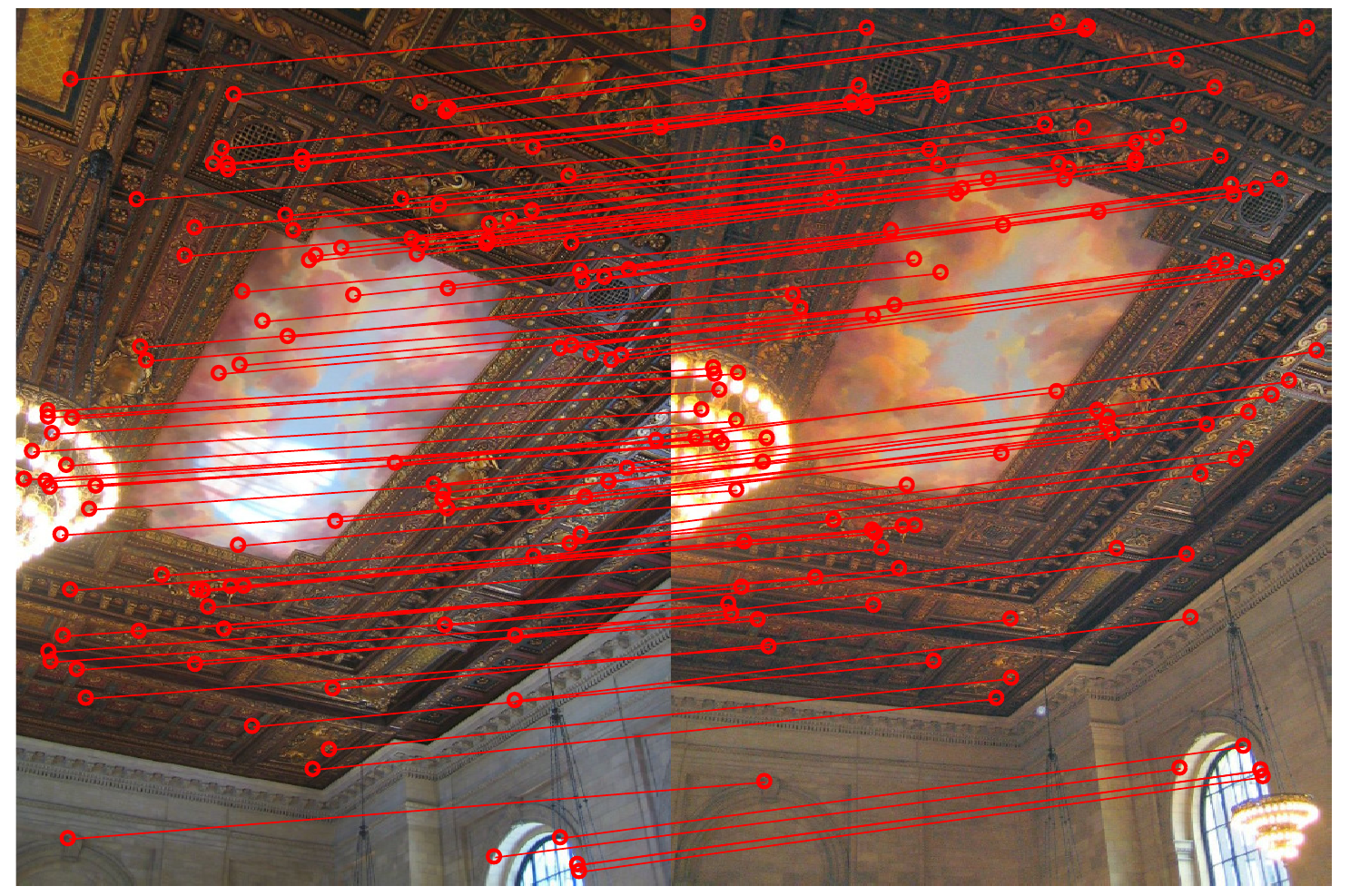}\label{subfig:Building2Corr}}
	\subfigure[FLRS consensus set (consensus: 321).]{\includegraphics[width=0.32\columnwidth]{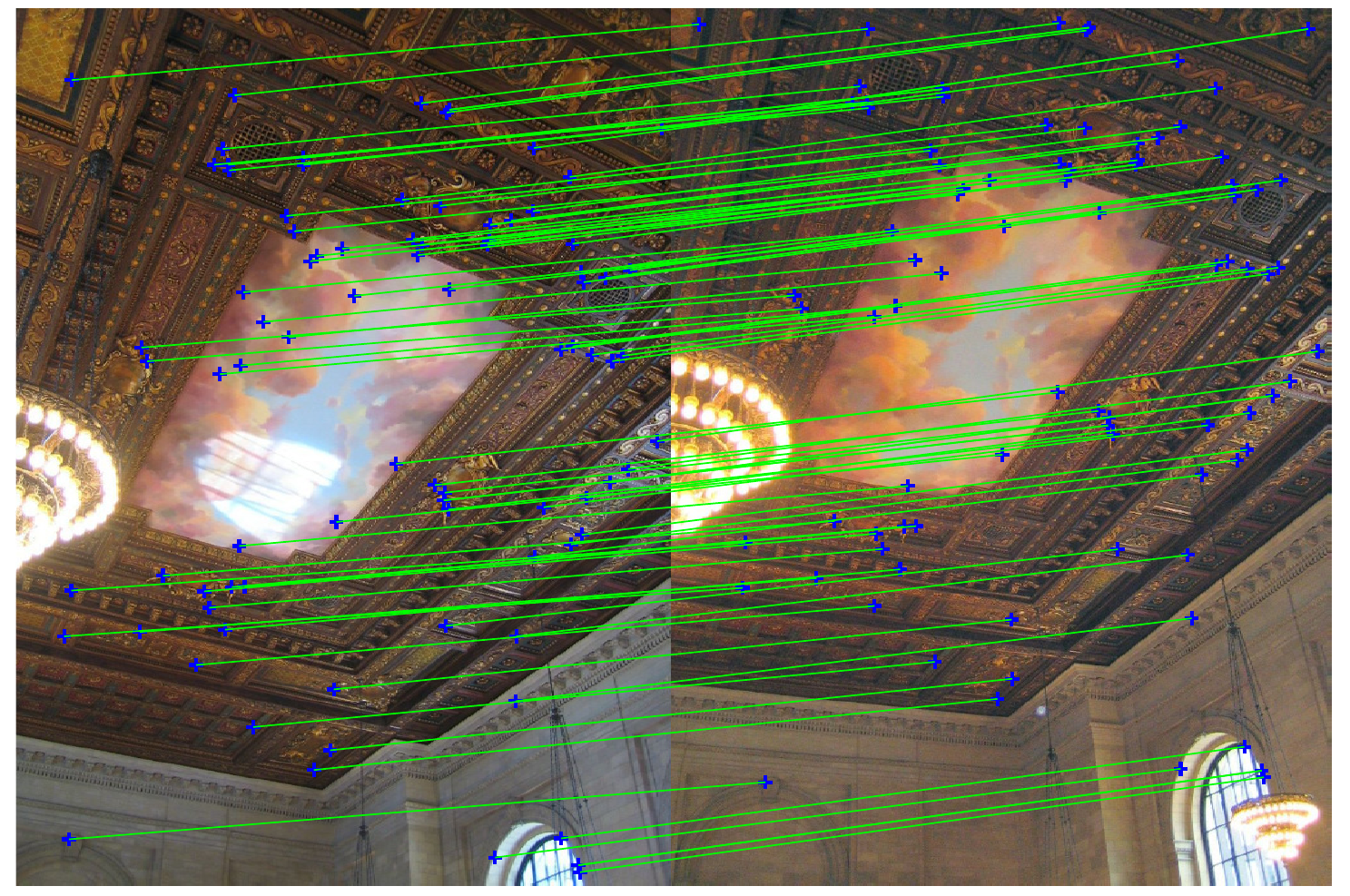}\label{subfig:Building2FLRS}}
		\subfigure[FLRS + IBCO consensus set (consensus: 331).]{\includegraphics[width=0.32\columnwidth]{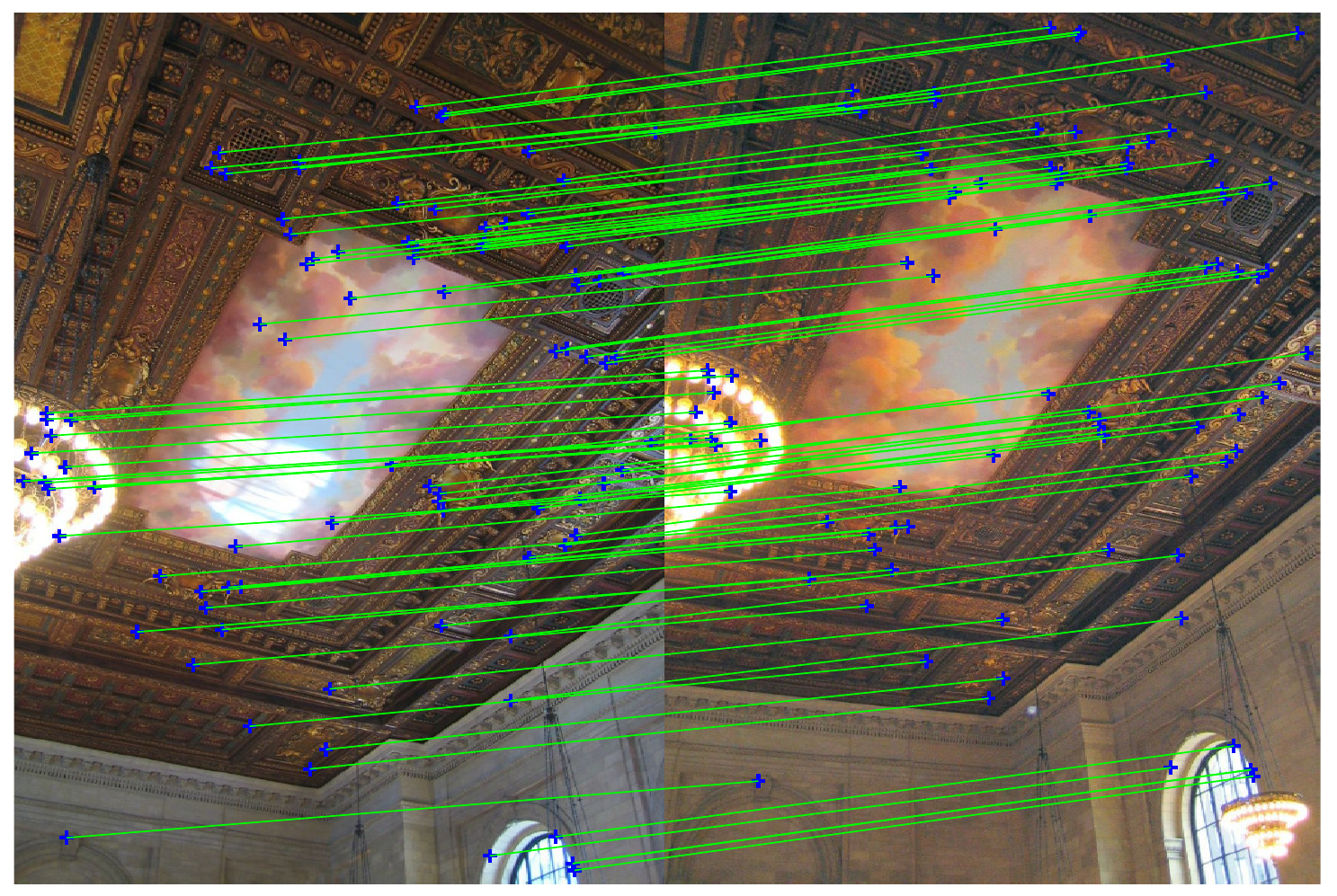}\label{subfig:Building2IBCO}}
	\caption{Data and results of robust homography estimation for \textit{Building1} (top) and \textit{Ceiling1} (bottom). Consensus sets were downsampled for visual clarity.}
	\label{fig:homoData}
\begin{minipage}{0.49\columnwidth}
	\subfigure[Average optimized consensus (as $\%$ of input size $N$). $N$ is provided in the brackets.]{\includegraphics[width=1\columnwidth]{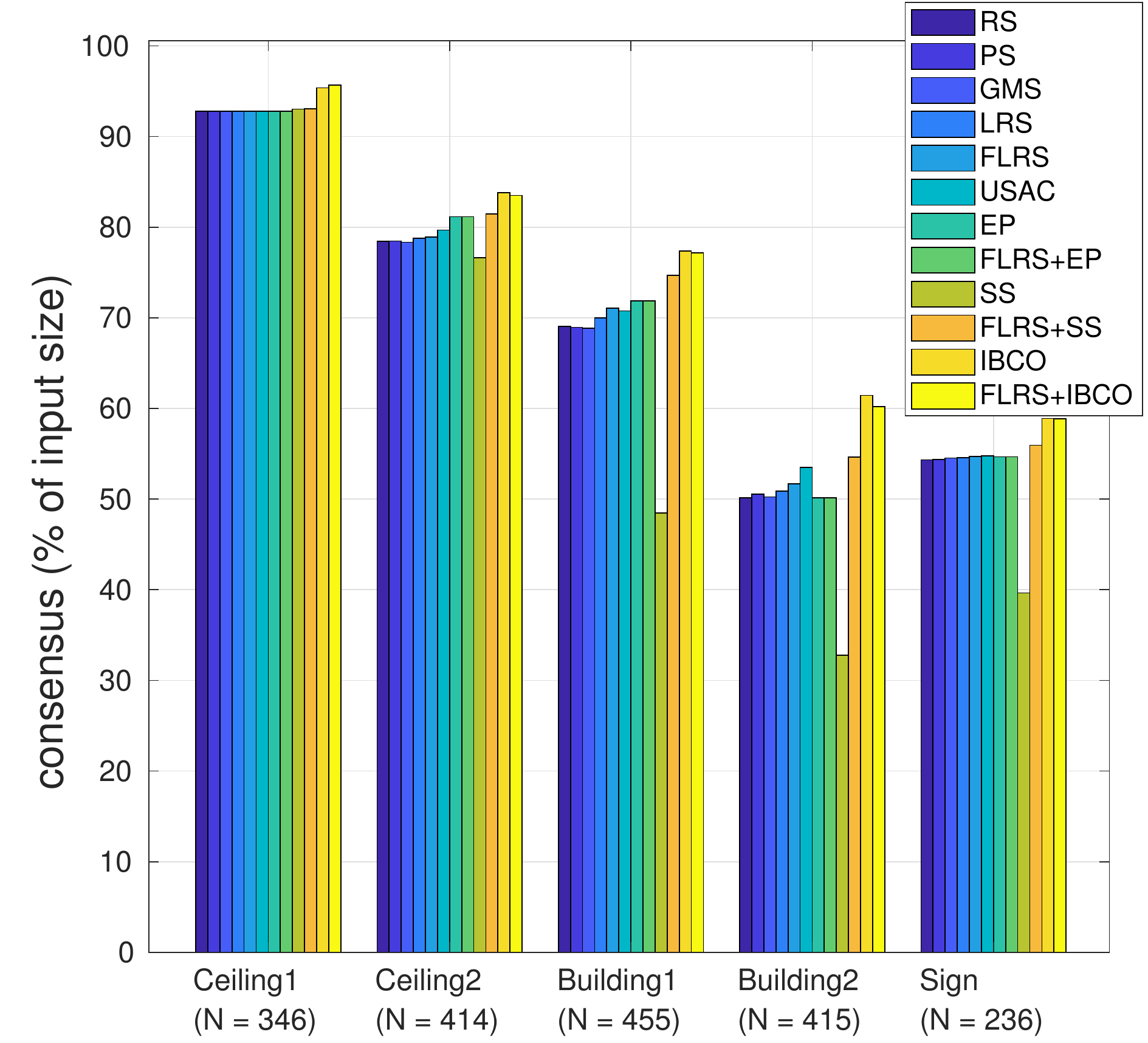}\label{subfig:HomoCon}}
\end{minipage}
\begin{minipage}{0.49\columnwidth}
	\subfigure[Standard deviation of optimized consensus over $50$ runs.]{\includegraphics[width=1\columnwidth]{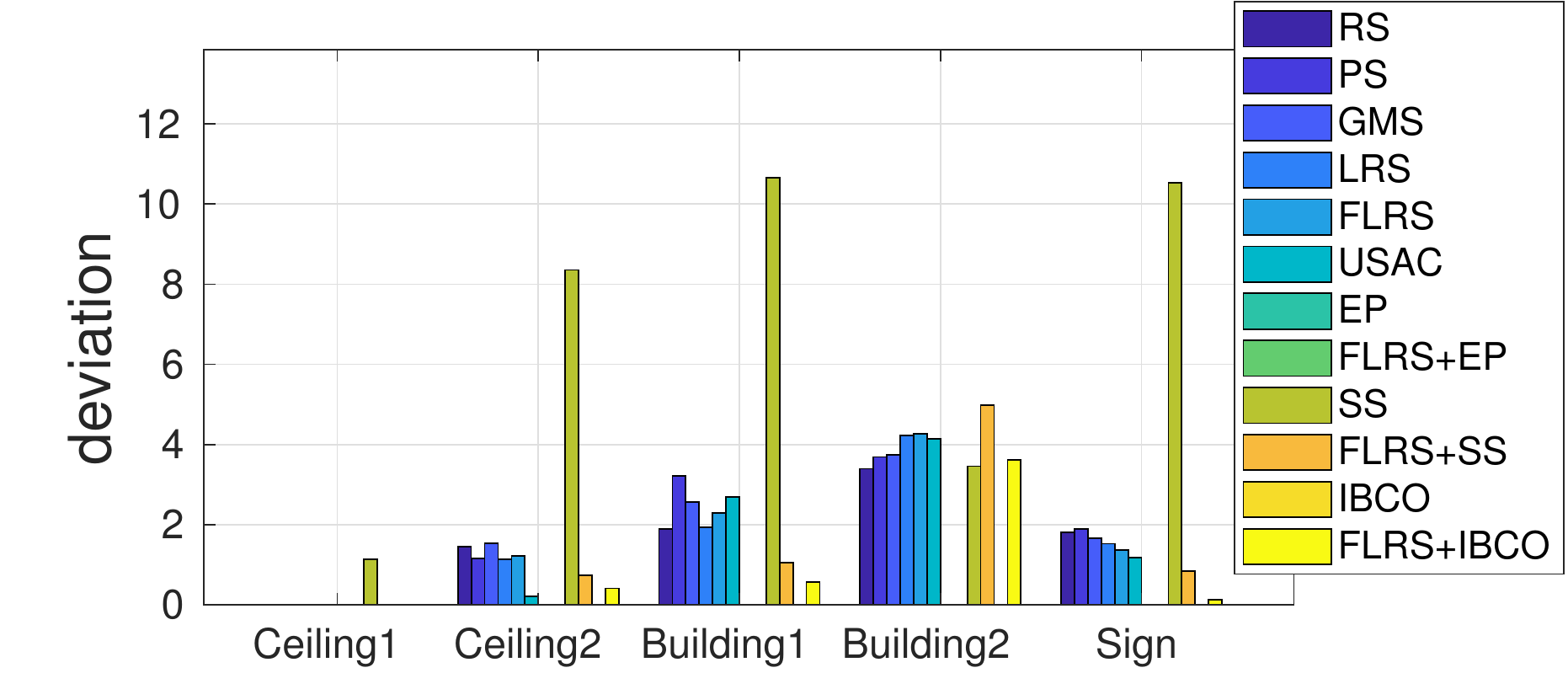}\label{subfig:HomoDev}}	
	\subfigure[Runtime in seconds.]{\includegraphics[width=1\columnwidth]{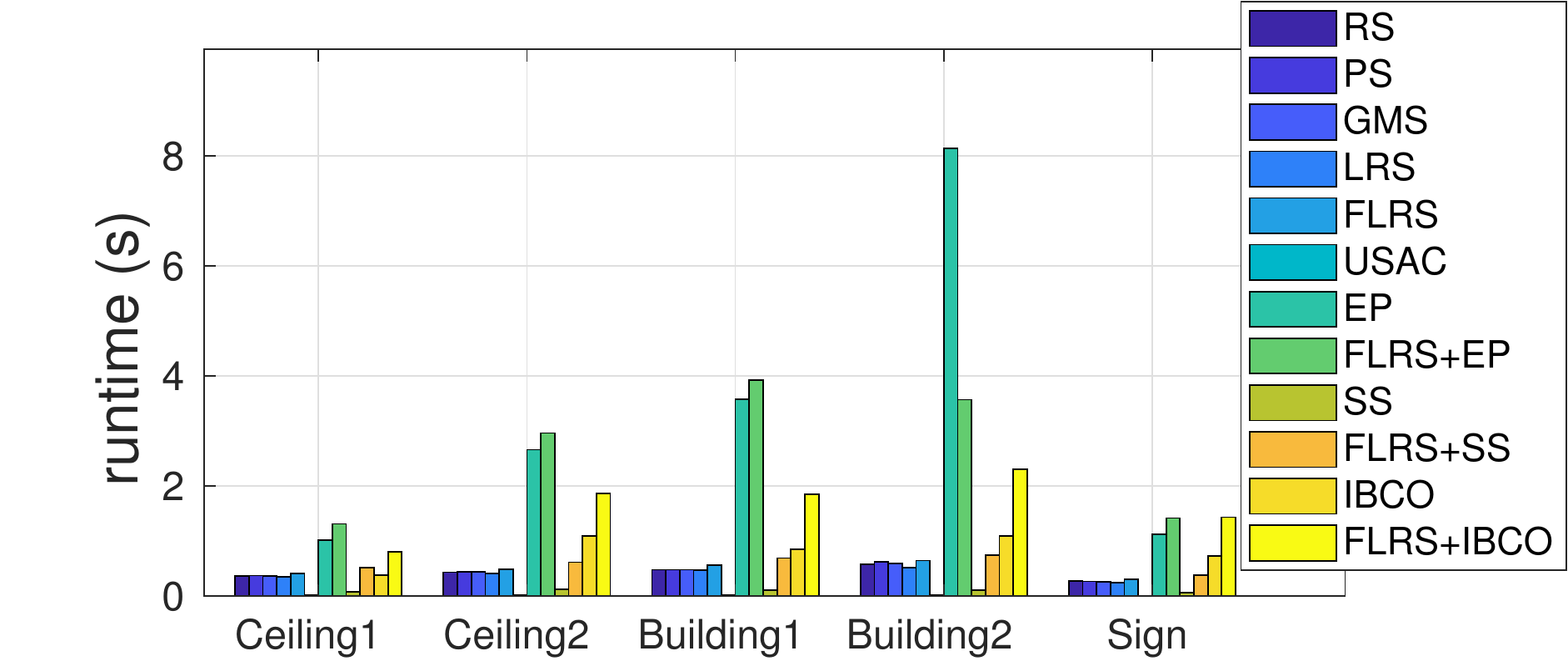}\label{subfig:HomoTime}}
\end{minipage}
	\caption{Robust homography estimation results.}
	\label{fig:HomoResult}
\end{figure}

Fig.~\ref{fig:HomoResult}, shows the quantitative results, averaged over 50 runs. Though marginally costlier than SS and random approaches, both IBCO variants found considerably larger consensus sets than other methods for all data. Meanwhile, different from the linear regression case, EP no longer had simiar result quality to IBCO. Also note that for challenging problems, e.g., \textit{Ceiling1} and \textit{Sign}, the two IBCO variants were the only methods that returned much higher consensus than RS.

\subsection{Triangulation}

Five feature tracks from the NotreDame dataset~\cite{snavely2006photo} were selected for triangulation, i.e., estimating the 3D coordinates. The input from each feature track contained a set of camera matrices and the corresponding 2D feature coordinates. The re-projection error was used as the distance measurement~\cite{ke2007quasiconvex} and the inlier threshold $\epsilon$ was set to 1 pixel. The size of minimal samples was 2 (views) for all RANSAC variants. The results are demonstrated in Fig.~\ref{fig:triResults}. For triangulation, the quality of the initial solution largely affected the performance of EP, SS and IBCO. Initialized with FLRS, IBCO managed to find much larger consensus sets than all other methods.

\subsection{Effectiveness of refinement}\label{sec:worseOff} 

Though all deterministic methods were provided with reliable initial FLRS solutions, IBCO was the only one that effectively refined all FLRS results. EP and SS sometimes even converged to worse than initial solutions. To illustrate these effects, Fig.~\ref{fig:WorseOffSol} shows the solution quality during the iterations of the three deterministic methods (initialized by FLRS) on \textit{Ceiling1} for homography estimation and \textit{Point 16} for triangulation. In contrast to EP and SS which progressively made the initial solution worse, IBCO steadily improved the initial solution.

It may be possible to rectify the behaviour of EP and SS by choosing more appropriate smoothing parameters and/or their annealing rates. However, the need for data-dependent tuning makes EP and SS less attractive than IBCO.

\subsection{Fundamental matrix estimation}\label{sec:exp_nonConv}
Image pairs from the two-view geometry corpus of CMP\footnote{\url{http://cmp.felk.cvut.cz/
data/geometry2view/}} were used for fundamental matrix estimation. As in homography estimation, SIFT correspondences were used as the input data. Since the Sampson error~\cite[Sec. 11.4.3]{hartley2003multiple} and the reprojection error~\cite[Sec. 11.4.1]{hartley2003multiple} for fundamental matrix estimation are not linear or quasiconvex, the deterministic algorithms (EP, SS, IBCO) cannot be directly applied. Thus, we linearize the epipolar constraint and use the algebraic error~\cite[Sec.~11.3]{hartley2003multiple} as the residual. The inlier threshold $\epsilon$ was set to $0.006$ for all data.

Further, a valid fundamental matrix satisfies the rank-2 constraint~\cite[Sec. 11.1.1]{hartley2003multiple}, which is non-convex. For EP, SS, IBCO, we impose the rank-2 constraint using SVD after each parameter vector updates (for IBCO, after each BCO run).

\begin{figure}[t]\centering
\begin{minipage}{0.49\columnwidth}
	\subfigure[Average optimized consensus (as $\%$ of input size $N$). $N$ is provided in the brackets.]{\includegraphics[width=1\columnwidth]{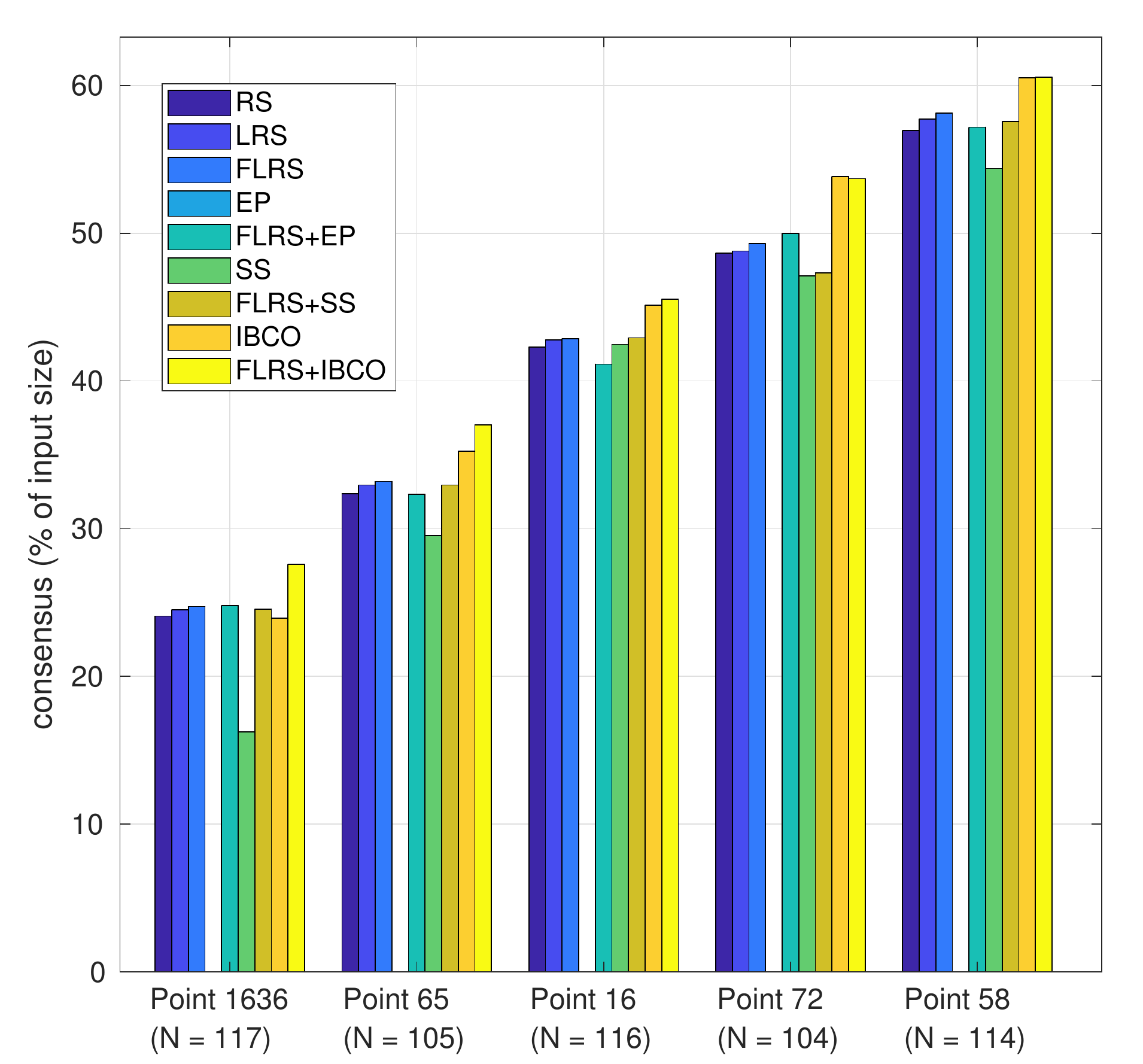}\label{subfig:triCon}}
\end{minipage}
\begin{minipage}{0.49\columnwidth}
	\subfigure[Standard deviation of optimized consensus over $50$ runs.]{\includegraphics[width=1\columnwidth]{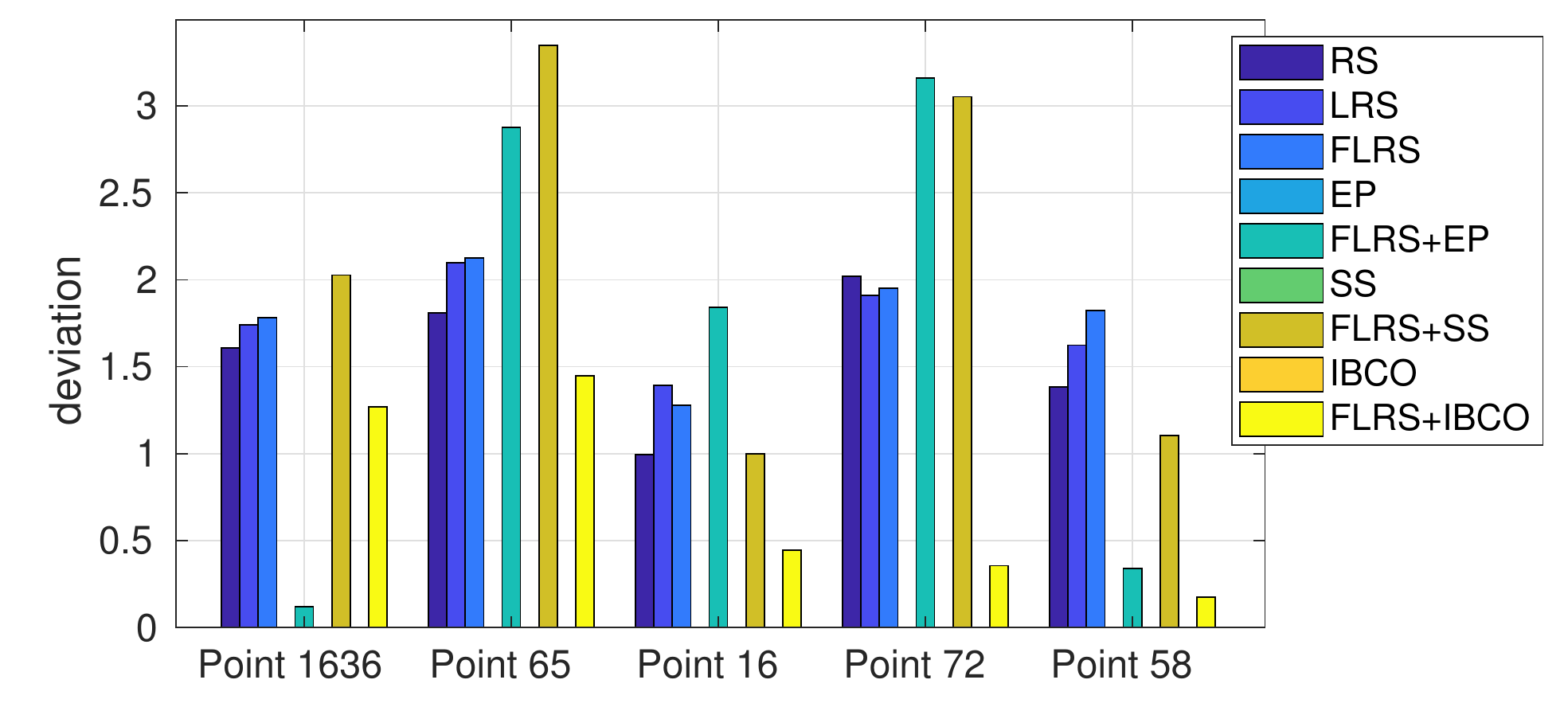}\label{subfig:triDev}}	
	\subfigure[Runtime in seconds.]{\includegraphics[width=1\columnwidth]{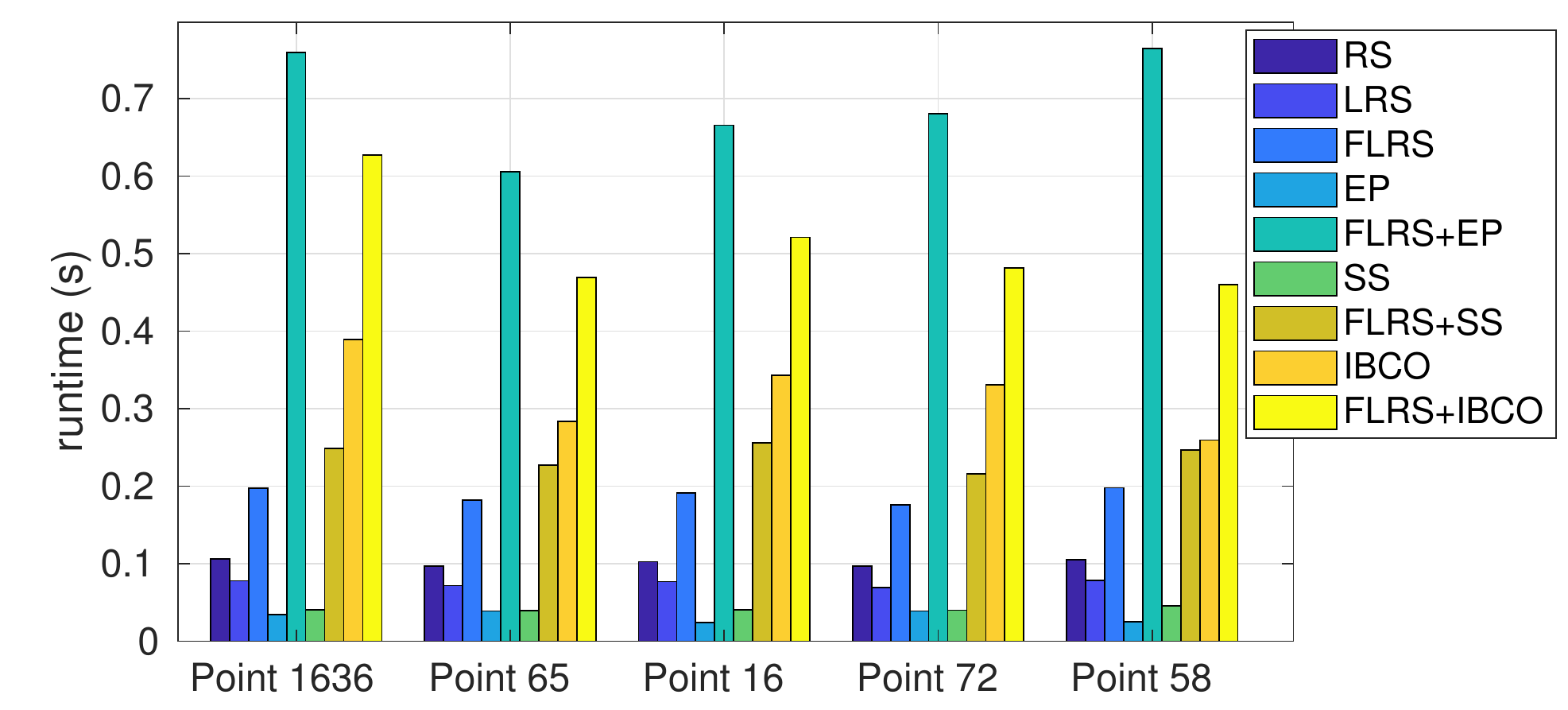}\label{subfig:triTime}}
\end{minipage}
	\caption{Robust triangulation results.}
	\label{fig:triResults}
\subfigure[\textit{Ceiling1} in homography estimation.]{\includegraphics[width=0.49\columnwidth]{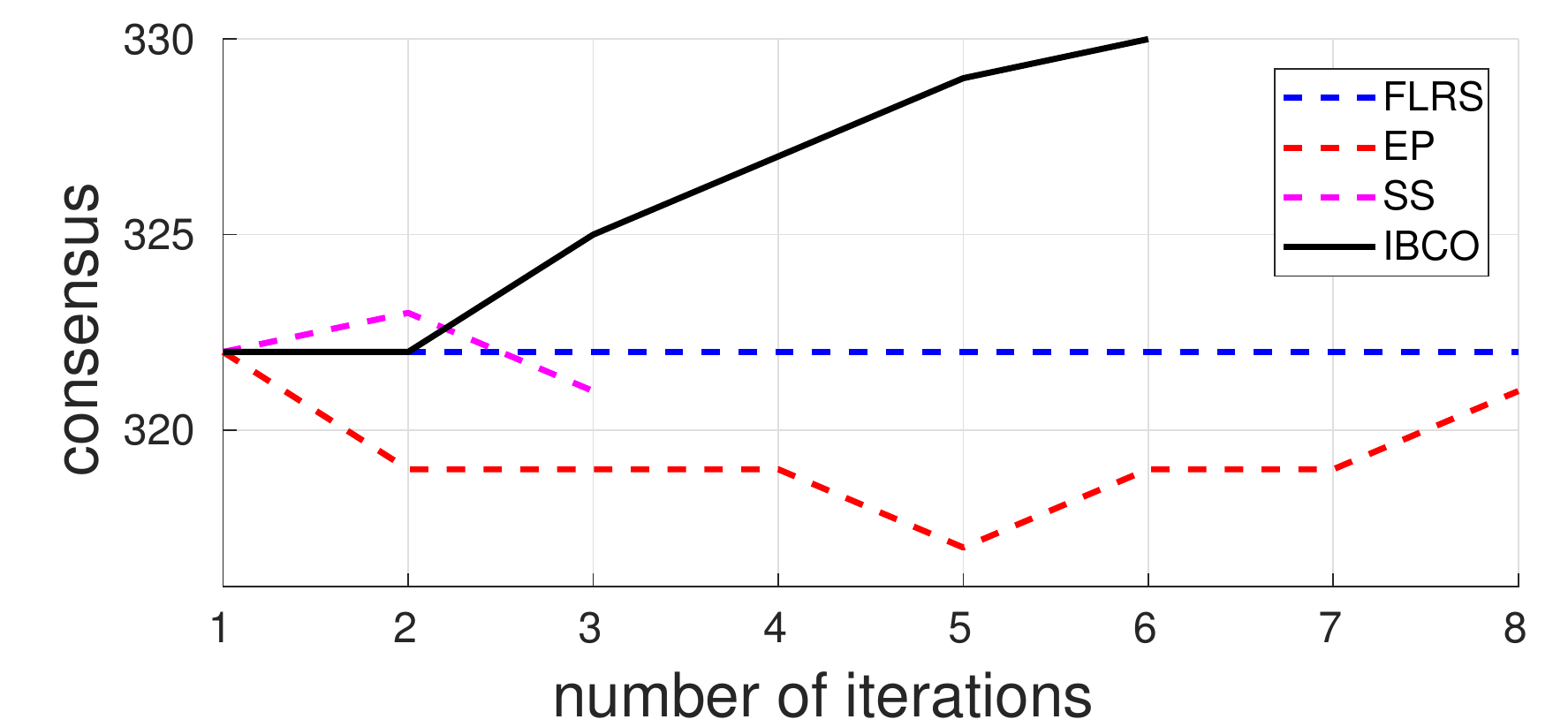}}
\subfigure[\textit{Point 16} in triangulation.]{\includegraphics[width=0.49\columnwidth]{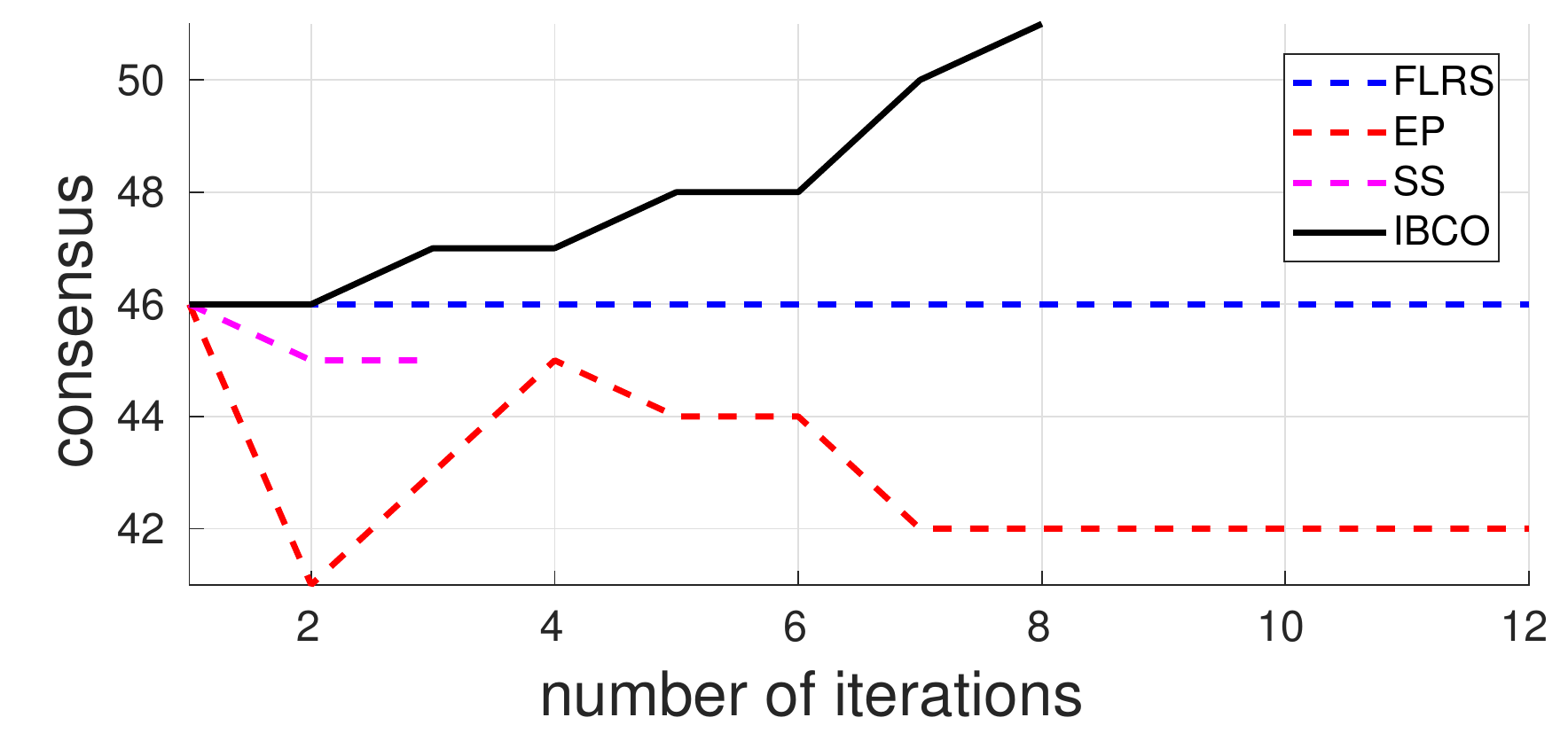}}
\caption{Consensus size in each iteration, given FLRS results as the initialization. Observe that EP and SS converged to worse off solutions.}\label{fig:WorseOffSol}
\end{figure}

Fig. \ref{fig:funData} depicts sample image pairs and generated SIFT correspondences, as well as consensus sets from FLRS and FLRS+IBCO. The seven-point method~\cite[Sec.~11.1.2]{hartley2003multiple} was used in USAC and the normalized 8-point algorithm~\cite[Sec.~11.2]{hartley2003multiple} was used in all other RANSAC variants. 

As shown in Fig.~\ref{subfig:FunCon}, unlike EP and SS who failed to refine the initial FLRS results for all the tested data, IBCO was still effective even though the problem contains non-convex constraints.

\begin{figure}[!htb]\centering
	\subfigure[Input correspondences ($N = 186$).]{\includegraphics[width=0.32\columnwidth]{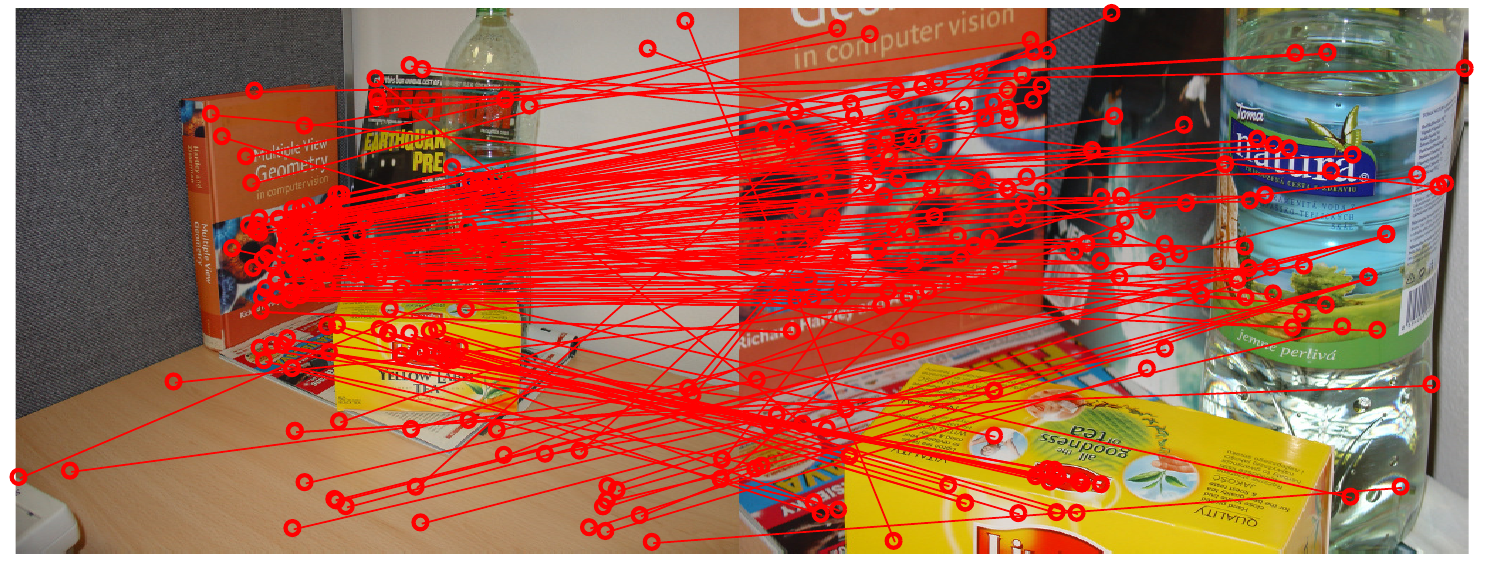}\label{subfig:zoomCorr}}
		\subfigure[FLRS consensus set (consensus: 85).]{\includegraphics[width=0.32\columnwidth]{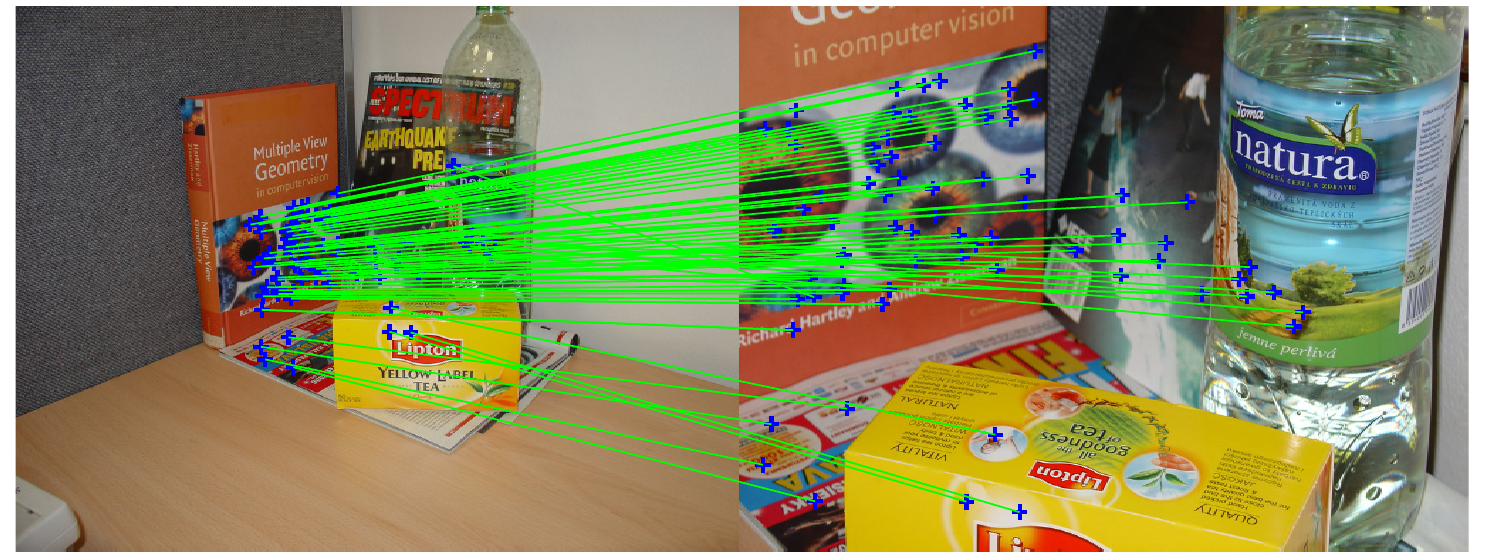}\label{subfig:zoomFLRS}}
		\subfigure[FLRS + IBCO consensus set (consensus: 97).]{\includegraphics[width=0.32\columnwidth]{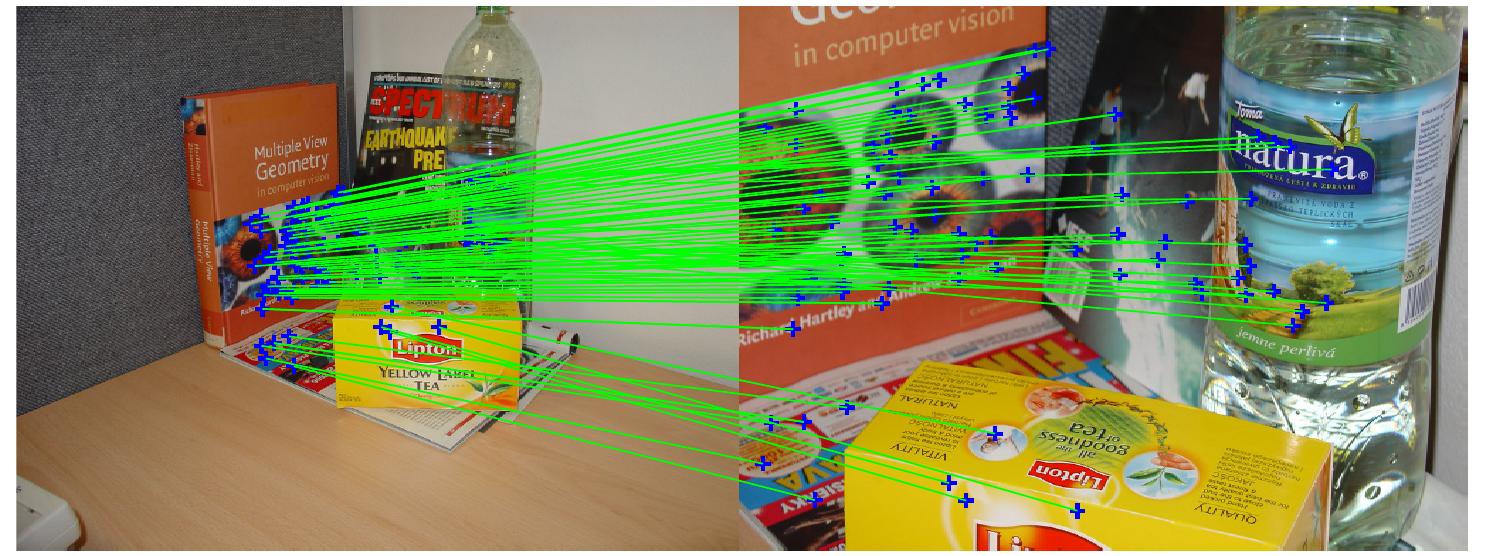}\label{subfig:zoomIBCO}}
		\subfigure[Input correspondences ($N = 101$).]{\includegraphics[width=0.32\columnwidth]{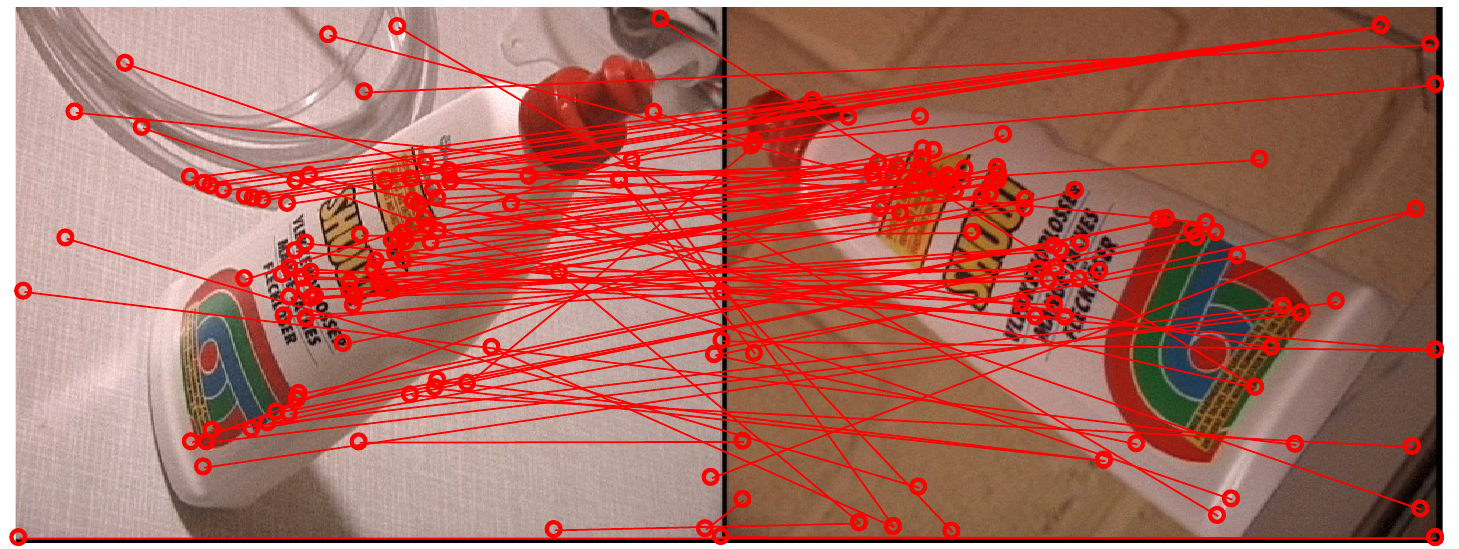}\label{subfig:shoutCorr}}
		\subfigure[FLRS consensus set (consensus: 32).]{\includegraphics[width=0.32\columnwidth]{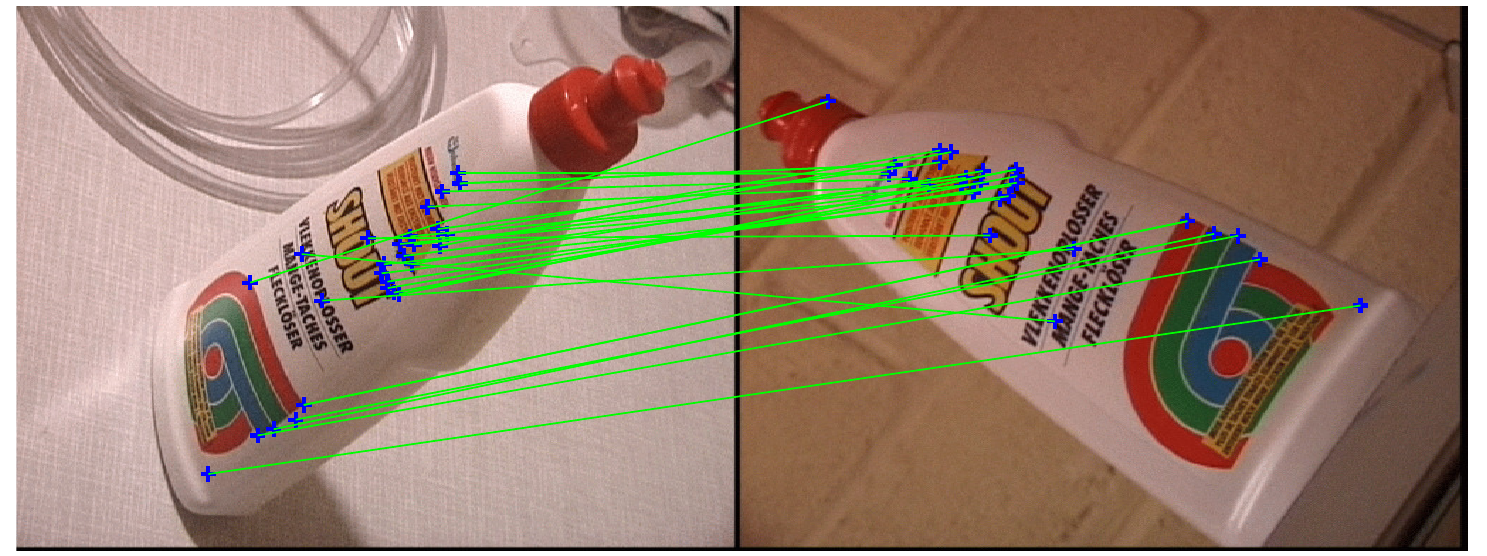}\label{subfig:shoutFLRS}}
		\subfigure[FLRS + IBCO consensus set (consensus: 36).]{\includegraphics[width=0.32\columnwidth]{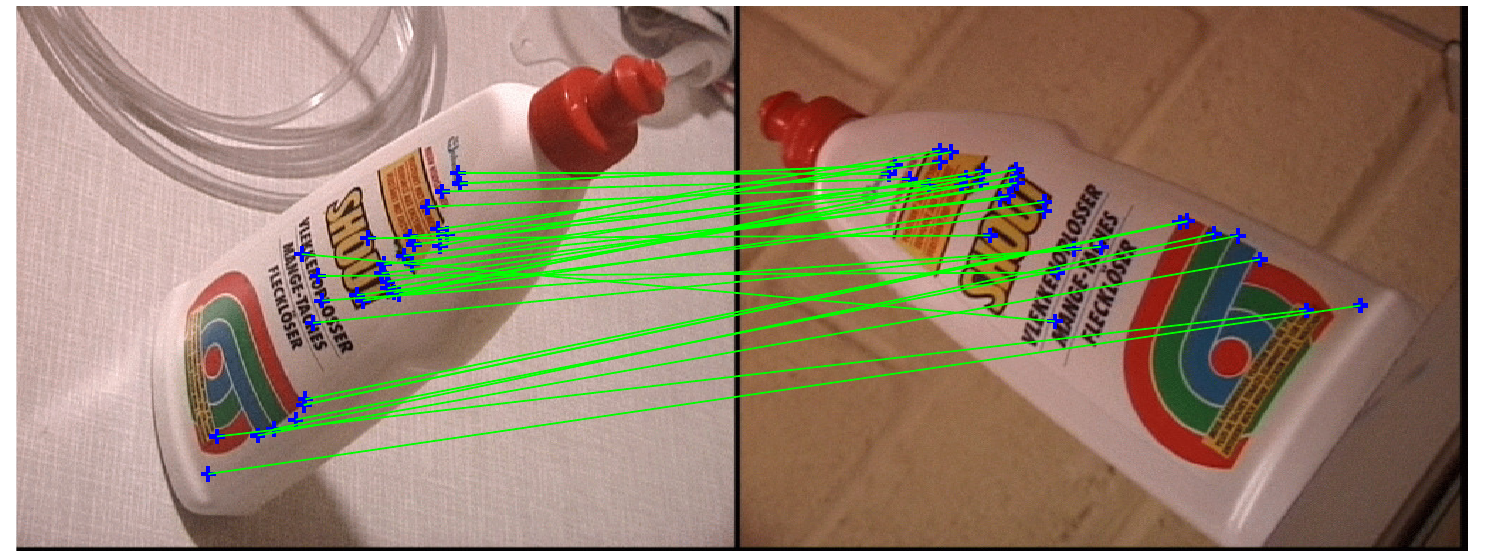}\label{subfig:shoutIBCO}}
	\caption{Data and results of fundamental matrix estimation for \textit{zoom} (top) and \textit{shout} (bottom).}
	\label{fig:funData}
\begin{minipage}{0.47\columnwidth}
	\subfigure[Average optimized consensus (as $\%$ of input size $N$). $N$ is provided in the brackets.]{\includegraphics[width=1\columnwidth]{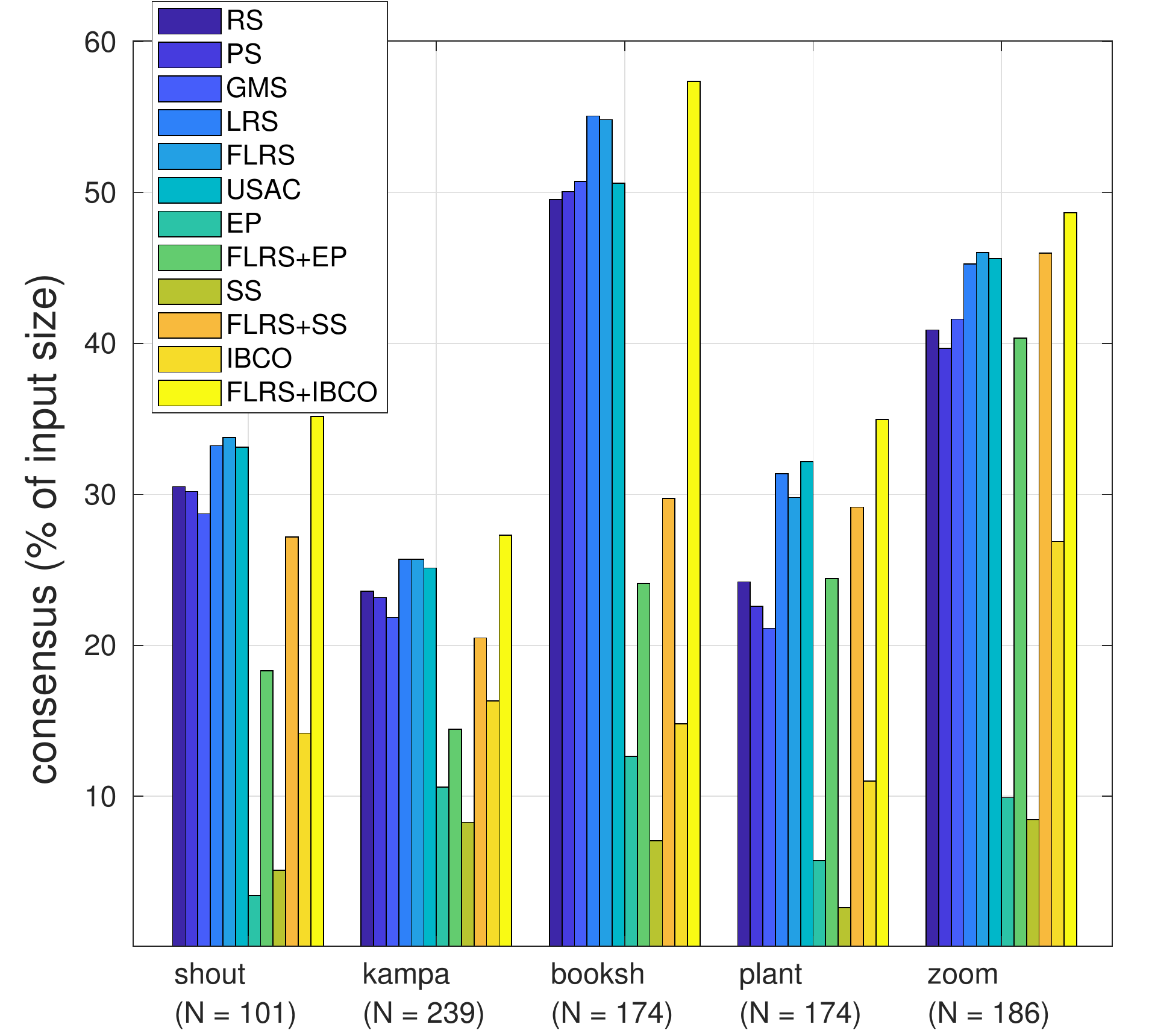}\label{subfig:FunCon}}
\end{minipage}
\begin{minipage}{0.47\columnwidth}
	\subfigure[Standard deviation of optimized consensus over $50$ runs.]{\includegraphics[width=1\columnwidth]{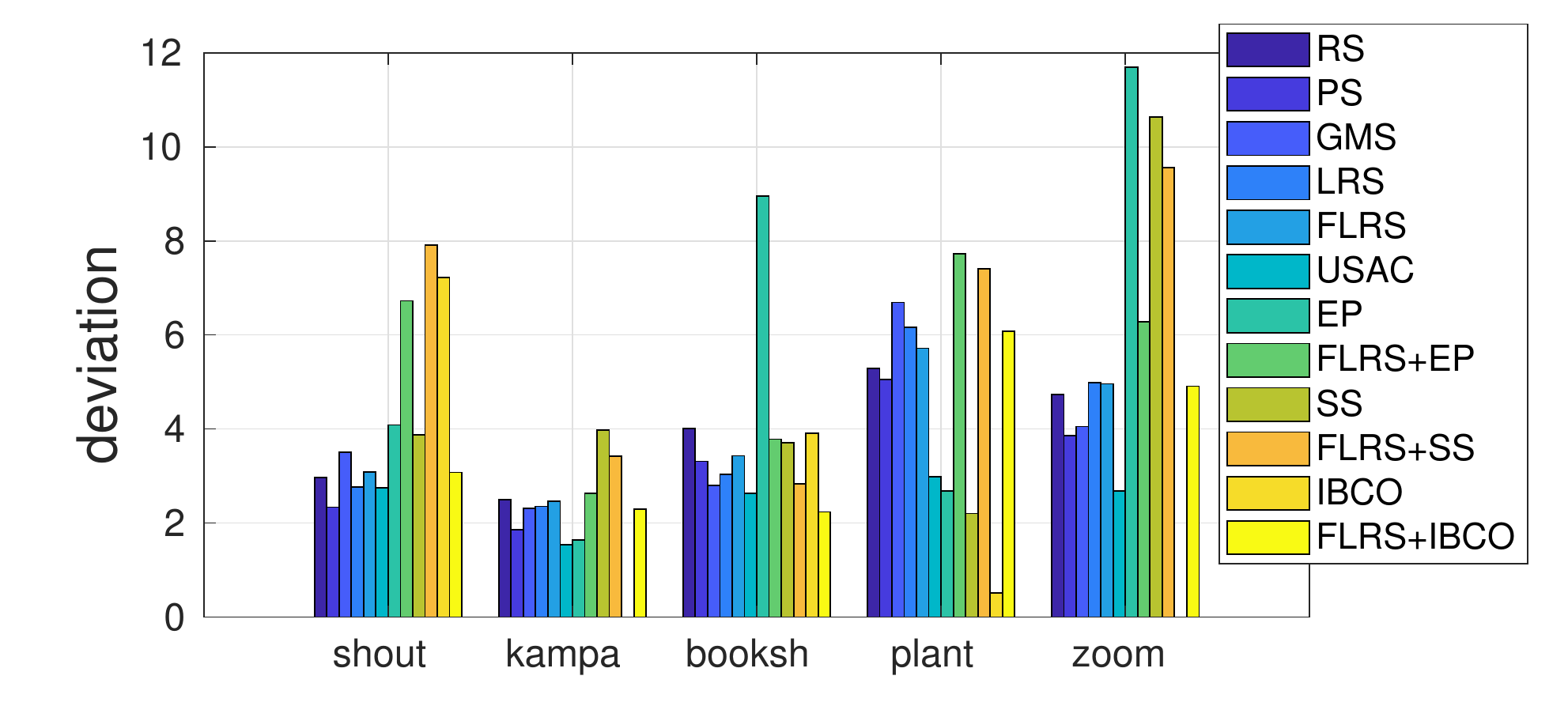}\label{subfig:FunDev}}	
	\subfigure[Runtime in seconds.]{\includegraphics[width=1\columnwidth]{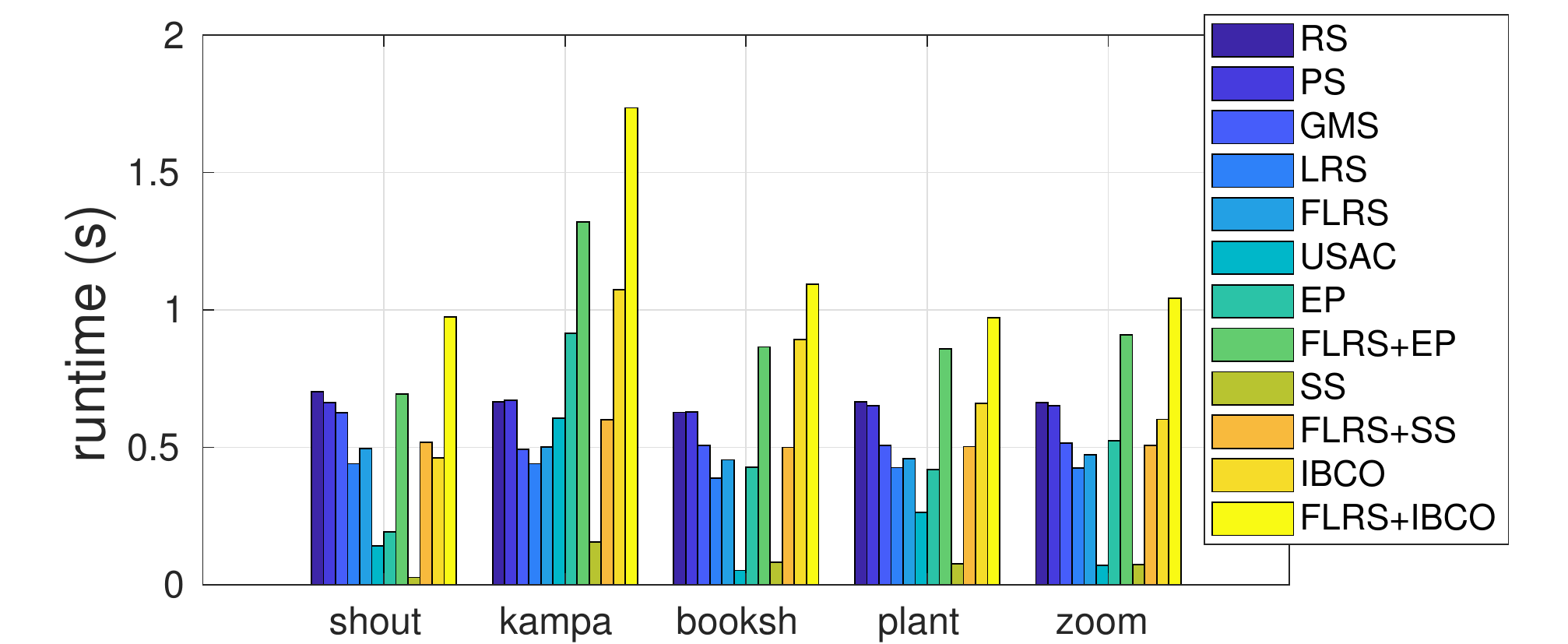}\label{subfig:FunTime}}
\end{minipage}
	\caption{Robust fundamental matrix estimation results.}
	\label{fig:FunResults}
\end{figure}

\section{Conclusions}
We proposed a novel deterministic algorithm for consensus maximization with non-linear residuals. The basis of our method lies in reformulating the decision version of consensus maximization into an instance of biconvex programming, which enables the use of bisection for efficient guided search. Compared to other deterministic methods, our method does not relax the objective of consensus maximization problem and is free from the tuning of smoothing parameters, which makes it much more effective at refining the initial solution. Experiments show that our method is able to greatly improve upon initial results from widely used random sampling heuristics.

\noindent{\textbf{Acknowledgements} This work was supported by the ARC grant DP160103490.}

\clearpage

\bibliographystyle{splncs04}
\bibliography{0685.bib}
\end{document}